\documentclass[mlmain]{jmlr}

\usepackage{multirow}
\usepackage{xcolor}
\definecolor{yellow}{HTML}{fefea2}
\definecolor{blue}{HTML}{00cdfa}
\definecolor{orange}{HTML}{ff9830}
\definecolor{green}{HTML}{5dbb63}
\usepackage{mathtools}
\usepackage{longtable}% for long tables

\usepackage{booktabs}
\usepackage{tikz}
\usepackage{wrapfig}
\usetikzlibrary{calc,arrows.meta, positioning, shapes.geometric, arrows}

\tikzstyle{theorem} = [rectangle, rounded corners, minimum width=3cm, minimum height=1cm, text centered, draw=black, fill=gray!10]
\tikzstyle{process} = [rectangle, minimum width=3cm, minimum height=1cm, text centered, draw=black, fill=blue!20]
\tikzstyle{arrow} = [thick,->,>=stealth]

\newtheorem{innercustomthm}{Theorem}
\newenvironment{theoremcopy}[1]{\renewcommand\theinnercustomthm{#1}\innercustomthm}{\endinnercustomthm}

\newtheorem{innercustomlemma}{Lemma}
\newenvironment{lemmacopy}[1]{\renewcommand\theinnercustomlemma{#1}\innercustomlemma}{\endinnercustomlemma}
\newtheorem{innercustomcorollary}{Corollary}
\newenvironment{corollarycopy}[1]{\renewcommand\theinnercustomcorollary{#1}\innercustomcorollary}{\endinnercustomcorollary}

\usepackage[textsize=tiny]{todonotes}
\usepackage[load-configurations=version-1]{siunitx} % newer version

\newcommand{\LL}{\text{L}}
\newcommand{\srg}{\textsf{srg}}
\theorembodyfont{\upshape}
\theoremheaderfont{\scshape}
\theorempostheader{:}
\theoremsep{\newline}

\jmlrvolume{}
\firstpageno{1}
\editors{Simone Azeglio, Christian Shewmake, Bahareh Tolooshams, Sophia Sanborn, Chase van de Geijin, Nina Miolane}

\jmlryear{2024}
\jmlrworkshop{Symmetry and Geometry in Neural Representations}

\title[Theoretical Insights into Line Graph Transformation on Graph Learning]{Theoretical Insights into Line Graph Transformation on Graph Learning}

  \author{\Name{Fan Yang\nametag{\thanks{Corresponding author: fan.yang@cs.ox.ac.uk}}} \Email{fan.yang@cs.ox.ac.uk} \\
  \Name{Xingyue Huang} \Email{xingyue.huang@cs.ox.ac.uk}\\
  \addr Department of Computer Science\\ University of Oxford, Oxford, UK.}

\begin{document}

\maketitle

\begin{abstract}
Line graph transformation has been widely studied in graph theory, where each node in a line graph corresponds to an edge in the original graph. This has inspired a series of graph neural networks (GNNs) applied to transformed line graphs, which have proven effective in various graph representation learning tasks. However, there is limited theoretical study on how line graph transformation affects the expressivity of GNN models. In this study, we focus on two types of graphs known to be challenging to the Weisfeiler-Leman (WL) tests: Cai-Fürer-Immerman (CFI) graphs and strongly regular graphs, and show that applying line graph transformation helps exclude these challenging graph properties, thus potentially assist WL tests in distinguishing these graphs.
 We empirically validate our findings by conducting a series of experiments that compare the accuracy and efficiency of graph isomorphism tests and GNNs on both line-transformed and original graphs across these graph structure types.
\end{abstract}

\section{Introduction}
\label{sec:intro}

Graph Neural Networks (GNNs) are prominent models widely studied in the domain of graph representation learning, with various applications ranging from social network analysis~\citep{fan2019graph} and recommendation systems~\citep{Baskin2016A} to drug discovery~\citep{Muzio2020Biological}. 
In particular, advancements in graph theory have played a significant role in the evolution of GNNs. The integration of graph measures, such as graph spectrum~\citep{baldesi2018spectral}, centrality~\citep{maurya2021graph}, and modularity~\citep{tsitsulin2023graph}, has shaped the design, analysis, and theoretical foundations of GNN research.

Line graph transformation is of particular interest due to its successful applications on GNNs, proving effective in several tasks such as link predictions~\citep{cai2020line,liu2021indigo}, community detection~\citep{chen2020supervised}, and material discovery~\citep{choudhary2021atomistic,nlgn_material}. 
As illustrated in Figure \ref{fig:line_illu}, the \emph{line graph transformation} is a graph rewriting algorithm that converts a graph’s edges into nodes. Transformed nodes are then connected if their corresponding edges in the original graph share a common node.
This transformation effectively captures the relationships between edges in 
a graph, and is particularly useful for tasks where edge-centric information is critical.

\begin{figure}[t]
\centering
\begin{tikzpicture}
    % First Graph
    \begin{scope}
        \node[draw, circle, fill = green] (a1) at (0,0) {$a$};
        \node[draw, circle, fill=blue] (b1) at (2,0) {$b$};
        \node[draw, circle, fill=yellow] (c1) at (1,2) {$c$};
        \node[draw, circle, fill=orange] (d1) at (3,2) {$d$};
        
        \draw (a1) -- (b1) ;
        \draw (a1) -- (c1) ;
        \draw (b1) -- (c1) ;
        \draw (b1) -- (d1) ;

        \node[draw=none,fill=none] (text) at (1.5,-1) {$G$};
    \end{scope}
    
    % Arrow from first to second graph
    \draw[->, thick] (4,1) -- (5,1);
    
    % Second Graph
    \begin{scope}[xshift=6cm]
        \node[draw, circle, fill=green] (a2) at (0,0) {$a$};
        \node[draw, circle,fill=blue] (b2) at (2,0) {$b$};
        \node[draw, circle, fill = yellow] (c2) at (1,2) {$c$};
        \node[draw, circle,fill=orange] (d2) at (3,2) {$d$};
        
        \draw (a2) -- (b2) node[midway, below] {$ab$};
        \draw (a2) -- (c2) node[midway, left] {$ac$};
        \draw (b2) -- (c2) node[midway, right, yshift=0.2cm] {$bc$};
        \draw (b2) -- (d2) node[midway, right] {$bd$};
        
        \draw[dashed] ($(a2)!0.5!(b2)$) to($(a2)!0.5!(c2)$);
        \draw[dashed] ($(a2)!0.5!(b2)$) to ($(b2)!0.5!(c2)$);
        \draw[dashed] ($(a2)!0.5!(c2)$) to ($(b2)!0.5!(c2)$);
        \draw[dashed] ($(b2)!0.5!(c2)$) to ($(b2)!0.5!(d2)$);
        \draw[dashed] ($(b2)!0.5!(a2)$) to ($(b2)!0.5!(d2)$);
    \end{scope}
    
    % Arrow from second to third graph
    \draw[->, thick] (10,1) -- (11,1);
    
    % Third Graph
    \begin{scope}[xshift=12cm]
        \node[draw, circle,scale = 0.8, path picture={
        \path[fill=yellow] (path picture bounding box.south) arc[start angle=-90, end angle=90, radius=.5] -- cycle;
        \path[fill=green] (path picture bounding box.north) arc[start angle=90, end angle=270, radius=.5] -- cycle;
    }] (ac3) at (0,1) {$ac$};
        \node[draw, circle,,scale = 0.8, path picture={
        \path[fill=blue] (path picture bounding box.south) arc[start angle=-90, end angle=90, radius=.5] -- cycle;
        \path[fill=green] (path picture bounding box.north) arc[start angle=90, end angle=270, radius=.5] -- cycle;
    }] (ab3) at (0.5,0) {$ab$};
        \node[draw, circle,,scale = 0.8,path picture={
        \path[fill=yellow] (path picture bounding box.south) arc[start angle=-90, end angle=90, radius=.5] -- cycle;
        \path[fill=blue] (path picture bounding box.north) arc[start angle=90, end angle=270, radius=.5] -- cycle;
    }] (bc3) at (1,1) {$bc$};
        \node[draw, circle,,scale = 0.8, path picture={
        \path[fill=cyan] (path picture bounding box.south) arc[start angle=-90, end angle=90, radius=.5] -- cycle;
        \path[fill=orange] (path picture bounding box.north) arc[start angle=90, end angle=270, radius=.5] -- cycle;
    }] (bd3) at (2,1) {$bd$};
        
        \draw (ac3) -- (bc3);
        \draw (ac3) -- (ab3);
        \draw (bc3) -- (ab3);
        \draw (bc3) -- (bd3);
        \draw (ab3) -- (bd3);

        \node[draw=none,fill=none] (text) at (1,-1) {$L(G)$};
    \end{scope}

\end{tikzpicture}
\vspace{-1em}
\caption{An example of converting a graph $G$ to its line graph $L(G)$.}
\label{fig:line_illu}
\end{figure}
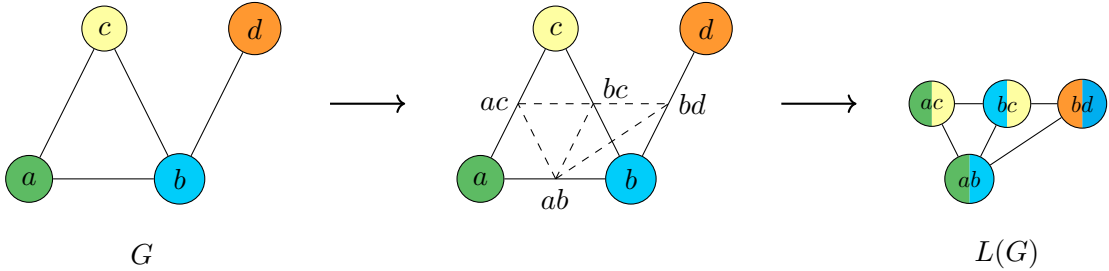

While the alignment of the Weisfeiler-Leman (WL) test to Message Passing Neural Networks (MPNNs)~\citep{morris2019weisfeiler,xu2018powerful} and higher-order GNNs~\citep{maron2020provablypowerfulgraphnetworks} on input graph has provided valuable insights into measuring the expressive power of these architectures, there has been limited exploration on the expressive power of these models applying on line graphs. Specifically, we are interested in answering the following questions:
\begin{enumerate}
    \item \emph{What are the connections between learning permutation-invariant functions on the line graph and the original graph?}
    \item \emph{What classes of graphs can line graph transformation help distinguish?}
\end{enumerate}

In this study, we aim to understand the effects of line graph transformations on the data from a graph theory perspective. We first observe that conducting a graph isomorphism test between a pair of connected graphs is equivalent to conducting it between their respective line graphs with the exception of the pair of graphs isomorphic to $K_{1,3}$ and $C_3$.  Moreover, this holds for repeated applications of line graph transformation. These observations inspire us to demonstrate the equivalence of the graph isomorphism test and uniform function approximation on graphs after an arbitrary number of transformations $L^{(n)}(G)$ and the original graph $G$. 

Furthermore, we shift our focus to Cai-Fürer-Immerman (CFI) graphs and strongly regular graphs, which are notably difficult for the WL test. We demonstrate that with several applications of line graph transformation, the resulting line graph no longer exhibits the same property, potentially simplifying the process of differentiation through standard isomorphism tests or GNNs.
We then perform a series of experiments to evaluate how line graph transformations affect the performance of standard isomorphism tests and GNNs in distinguishing difficult graph classes, focusing on the effectiveness of applying lower-order WL tests on the transformed line graphs for isomorphism checks.

Our contribution can be summarized as below:
\begin{enumerate}
    \item We show that the universal approximation of permutation-invariant functions on connected graphs after arbitrary numbers of line graph transformation is equivalent to their original graph, assuming the root graph is not isomorphic to $K_{1,3}$. 
    \item We study two challenging graph types for isomorphism tests: CFI graphs and strongly regular graphs, and show that they are excluded in line graphs with few exceptions. 
    \item We empirically validate our theory by conducting a series of graph isomorphism tests and GNNs and present line graph transformation's effect on distinguishing challenging graph pairs.
\end{enumerate}

All proof details are presented in Appendix \ref{apd:proof}.

\section{Related works}

\paragraph{GNNs and WL test}

The comparison of the expressive power of GNNs with those of the WL tests is a well-studied topic in the literature. \citet{xu2018powerful} and \citet{morris2019weisfeiler} demonstrated that Message-Passing Neural Networks (MPNNs) are capable of matching the expressiveness of the 1-WL test. Furthermore, \citet{morris2019weisfeiler} extended this analysis by developing $k$-dimensional GNNs ($k$-GNNs) to align with the expressiveness of the $k$-WL tests, while inheriting their limitations on computational complexity. This led to theoretical analysis of several other existing classes of higher-order GNNs to study their expressive power. For instance, \citet{frasca2022understanding} has drawn connections from subgraph GNNs to Invariant Graph Networks (IGNs)~\citep{Maron2018InvariantAE}, showing that the expressive power of subgraph GNNs is inherently limited by 3-WL.
\citet{maron2020provablypowerfulgraphnetworks} also derived architectures that closely match the expressive power of 3-WL.

\paragraph{Challenging graphs} Another line of works~\citep{babai1979kwl, bouritsas2020srg3wl, cai1992optimal} focuses on understanding what classes of graphs are inherently challenging to the hierarchy of WL tests, shown in Figure \ref{fig:wl_relationship}.
One of the well-known classes is the Cai-Fürer-Immerman (CFI) graphs \citep{cai1992optimal}, designed specifically to be indistinguishable from the standard WL algorithm. CFI graphs are built by modifying the edges of a base graph with additional parity information encoded into gadgets, such that there exists a CFI graph that $k$-WL indistinguishable from any fixed $k$.

Another class of counter-examples to the WL test comes from the class of regular graphs, often characterized by their highly symmetrical structure and uniform node degrees, thus resulting in the same color refinement by 1-WL algorithms. 
In particular, strongly regular graphs include properties such as consistent numbers of common neighbors for adjacent and non-adjacent node pairs and are known to be 3-WL indistinguishable \citep{arvind2018subgraphgnn}, thus presenting challenges even with higher-order GNNs.

\begin{figure}
    \centering
\includegraphics[width=0.9\textwidth]{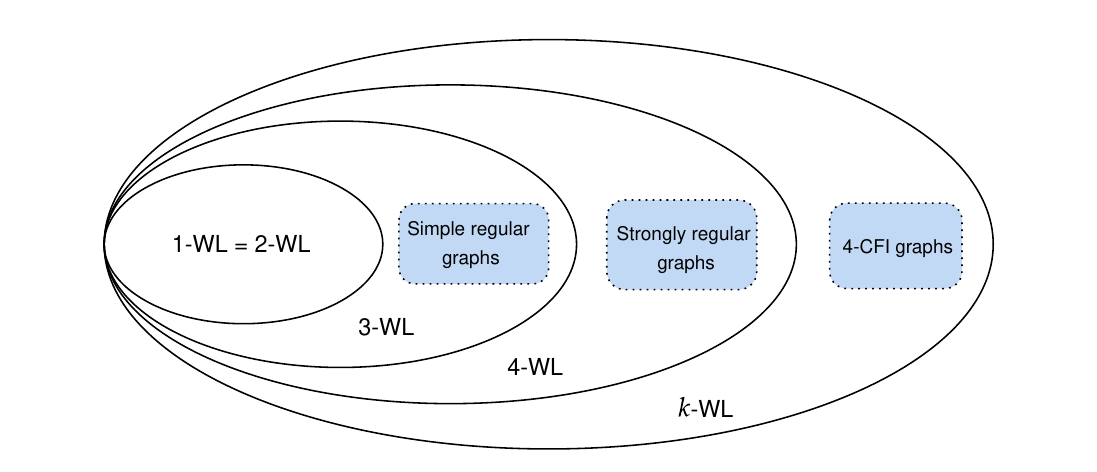}
    \caption{Relationships between WL tests and challenging graph types.}
    \label{fig:wl_relationship}
\end{figure}

\paragraph{Line Graph Neural Networks}

While there does not yet exist a comprehensive theoretical justification for the use of line graphs with GNNs, there have been some studies showing promising experimental results with line graph neural networks.
\citet{cai2020line} show that line graphs effectively convert link prediction problems into node classification problems, which avoids the information loss caused by necessary graph pooling and leads to competitive performance in various link prediction benchmarks. In materials science, line graphs have been used for explicit modeling of chemical bond angles, which has proven useful in atomic property predictions ~\citep{choudhary2021atomistic}. Moreover, \citet{chen2020supervised} demonstrated that line graph neural networks can construct the non-backtracking operator, enhancing the community detection capability in sparse graphs.

\section{Line graphs and their properties}
\paragraph{Notations}
We write ${G} = (V, E)$ to represent a finite, undirected, simple graph where $V$ is the set of nodes and $E$ is the set of edges. We denote $V(G)$ and $E(G)$ to specify the graph to which the set of nodes and the set of edges belong.
Given a node $u \in V(G)$, $d_G(u)$ denotes the degree of $u$ on graph $G$. 
We write $G\cong H $ if a graph $G$ is isomorphic to $H$ and $G\ncong H$  otherwise.
We also adopt the common notations for special graphs, e.g., $P_n,$ $C_n,$ and $K_n$ represent a path, a cycle, and a complete graph with $n$ nodes, respectively. We also write $K_{1,n}$ to represent star graphs with one center node and $n$ leaves. 
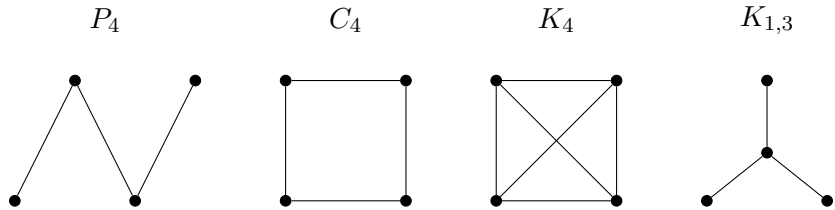
\begin{figure}[h]
\centering
\begin{tikzpicture}[scale = 0.8]
    % First Graph
    \begin{scope}
        \node[draw, circle, fill=black, minimum size=4pt, inner sep=0pt] (a1) at (0,0) {};
        \node[draw, circle, fill=black, minimum size=4pt, inner sep=0pt] (b1) at (2,0) {};
        \node[draw, circle, fill=black, minimum size=4pt, inner sep=0pt] (c1) at (1,2) {};
        \node[draw, circle, fill=black, minimum size=4pt, inner sep=0pt] (d1) at (3,2) {};
        
        \draw (a1) -- (c1) node[midway, left] {};
        \draw (b1) -- (c1) node[midway, right] {};
        \draw (d1) -- (b1) node[midway, right] {};
        \node[draw=none, fill=none] (text) at (1.5,3) {$P_{4}$};
    \end{scope}
    
    % Arrow from first to second graph
    \begin{scope}[xshift=4.5cm]
        \node[draw, circle, fill=black, minimum size=4pt, inner sep=0pt] (a1) at (0,0) {};
        \node[draw, circle, fill=black, minimum size=4pt, inner sep=0pt] (b1) at (2,0) {};
        \node[draw, circle, fill=black, minimum size=4pt, inner sep=0pt] (c1) at (2,2) {};

        \node[draw, circle, fill=black, minimum size=4pt, inner sep=0pt] (d1) at (0,2) {};
        
        \draw (a1) -- (b1) node[midway, left] {};
        \draw (b1) -- (c1) node[midway, right] {};
        \draw (c1) -- (d1) node[midway, right] {};
        \draw (d1) -- (a1) node[midway, right] {};
        \node[draw=none, fill=none] (text) at (1,3) {$C_{4}$};
    \end{scope}
    
    \begin{scope}[xshift=8cm]
        \node[draw, circle, fill=black, minimum size=4pt, inner sep=0pt] (a1) at (0,0) {};
        \node[draw, circle, fill=black, minimum size=4pt, inner sep=0pt] (b1) at (2,0) {};
        \node[draw, circle, fill=black, minimum size=4pt, inner sep=0pt] (c1) at (2,2) {};

        \node[draw, circle, fill=black, minimum size=4pt, inner sep=0pt] (d1) at (0,2) {};
        
        \draw (a1) -- (b1) node[midway, left] {};
        \draw (b1) -- (c1) node[midway, right] {};
        \draw (c1) -- (d1) node[midway, right] {};
        \draw (d1) -- (a1) node[midway, right] {};
        \draw (c1) -- (a1) node[midway, right] {};
        \draw (b1) -- (d1) node[midway, right] {};
         \node[draw=none, fill=none] (text) at (1,3) {$K_{4}$};
    \end{scope}
    
    % Third Graph
    \begin{scope}[xshift=11.5cm]
        \node[draw, circle, fill=black, minimum size=4pt, inner sep=0pt] (a1) at (0,0) {};
        \node[draw, circle, fill=black, minimum size=4pt, inner sep=0pt] (b1) at (2,0) {};

        \node[draw, circle, fill=black, minimum size=4pt, inner sep=0pt] (d1) at (1,2) {};
        \node[draw, circle, fill=black, minimum size=4pt, inner sep=0pt] (e1) at (1,0.8) {};
        
        \draw (e1) -- (a1) node[midway, left] {};
        \draw (e1) -- (d1) node[midway, right] {};
        \draw (e1) -- (b1) node[midway, right] {};

         \node[draw=none, fill=none] (text) at (1,3) {$K_{1,3}$};
    \end{scope}

\end{tikzpicture}
\caption{Examples of graph structures.
}
\label{fig:graph_examples}
\end{figure}

\paragraph{Line graph} 

A line graph $\LL(G)$ is derived from another graph $G$ termed the \emph{root graph}. Formally, given a root graph \(G = (V, E)\), where \(V\) is the set of nodes and \(E\) is the set of edges, the \emph{line graph}, denoted as L\((G)\), is constructed such that each node in L\((G)\) corresponds to an edge in \(E\), and two nodes in L\((G)\) are adjacent if their corresponding edges in \(G\) are adjacent with the same node in \(V\). For notation convenience, we use $\LL^{(n)}(G)$ to represent the resulting graph after $n$ times line graph transformation (e.g. $\LL^{(2)}(G) \coloneqq \LL(\LL(G))).$
The definition of line graphs induces some corollary properties. For instance, let $G$ be a simple undirected graph, it follows naturally that $|E(\text{L}(G))| = |V(G)|$. In addition, Lemma \ref{degree_line_graph} introduces a useful formula when analyzing the degree of new graphs.

\begin{lemma}\label{degree_line_graph}Let $u, v\in V(G)$ such that they are adjacent by an edge $e\in E(G).$ The edge $e$'s corresponding node representation $w_e\in V(\text{L}(G))$ follows \(
    d_{\LL(G)}(w_e) = d_G(u)+d_G(v)-2.\)
\end{lemma}

Some noticeable line graph transformations include $\LL(P_n) = P_{n-1}$ (paths to shorter paths), $\LL(C_n) = C_n$ (cycles to cycles), $\LL(K_{1,n}) = K_n$ (stars to complete graphs). In particular, a \emph{claw graph}, also known as the complete bipartite graph $K_{1,3},$ has a line graph of $C_3$. \citet{van1965interchange} showed that for a connected graph $G$, the repeated line graph transformations $\LL^{(n)}(G)$ has unbounded size as $n\rightarrow\infty$ if and only if $G$ is not a path, cycle or $K_{1,3}$. Additionally, we say a graph $G$ is \emph{triangle-containing} if $G$ has a subgraph $C_3$, and is \emph{triangle-free} if it does not.

\paragraph{Whitney's Isomorphism Theorem} 
Formulated by \citet{whitney1992congruent}, a clear relation between original graphs and their line graphs is established. Whitney's theorem demonstrates a one-to-one mapping between the line graph and its root graph except for exceptions of $C_3$ and $K_{1,3}$, providing the theoretical foundation for learning with line graphs.

\begin{theorem}[Whitney's Isomorphism Theorem  \citep{whitney1992congruent}]
\label{whitney}Let \( G \) and \( H \) be finite, connected graphs. Then \( G \) is isomorphic to \( H \) if and only if their line graphs are isomorphic, with the exception of the case where \( G \) and \( H \) is a pair of $C_3$ and $K_{1,3}$, in which case their line graphs are both isomorphic to \( C_3 \).
\end{theorem}

\paragraph{Beineke's Forbidden Induced Subgraphs}
In addition to the ``almost injective'' relationship between the root graph and its line graph transformation, there is an additional constraint on the constructed line graph.
\citet{beineke1970characterizations} introduced the characteristic of line graphs that a graph \(G\) is a line graph of another graph if and only if \(G\) does not contain any of Beineke's forbidden subgraphs as an induced subgraph (See Appendix \ref{app:beineke} for more details). As a direct consequence, we have the following key corollary, which will be instrumental in the next section.
\begin{corollary}
\label{thm:claw-free}
    Let $G$ be a simple and undirected graph, then $\LL(G)$ does not contain $K_{1,3}$ as an induced subgraph.
\end{corollary}

\section{Theoretical Framework}

We start by presenting the theoretical framework overview of this study, summarized as follows:
\begin{enumerate}
    \item We demonstrate the equivalence between isomorphism testing on line graphs after any number of transformations and the approximation of permutation-invariant functions on the connected root graph with the exception of the $K_{1,3}$ graph. (Theorem \ref{thm:equivalence}, Corollary \ref{cor:n_step_linetransform})
    \item We show that higher orders of CFI construction, as proposed by \citet{cai1992optimal}, are excluded from the set of line graphs. (Theorem \ref{thm:cfi})
    \item We prove that with the exception of $C_n$ when $3 \leq n \leq 5$, connected strongly regular graphs are transformed into non-strongly regular graphs by at most two line graph transformations. (Lemma \ref{lmm:k_small}, Theorem \ref{thm:tri-line-graph-regularity}, Theorem \ref{main_thm})
    \item We extend the theorem on equivalence and strongly regular graphs to disconnected graphs. (Corollary \ref{cor:disconnected_equiv}, Corollary \ref{cor:disconnected_srg})
\end{enumerate}
\subsection{Equivalence}

In this section, we first focus on connected graphs that are not isomorphic to the claw graph $K_{1,3}$. Theorem \ref{thm:equivalence} shows that approximating permutation-invariant functions on root graph $G$ is \emph{equivalent} to that on line graph $L(G)$ after a single transformation.

\begin{theorem}
\label{thm:equivalence}
 Let \(\mathcal{G}\) be a set of connected non-claw graphs and \(\mathcal{C}\) be a collection of functions, such that $\forall G_1,G_2\in\mathcal{G}$ such that \(G_1\ncong G_2\), \(\exists h\in \mathcal{C}\), \(h(\LL(G_1))\neq h(\LL(G_2))\). Then, \(\mathcal{C}\) can universally approximate any permutation-invariant function $f:\mathcal{G}\rightarrow \mathbb{R}$.
\end{theorem}
 The proof idea is to leverage Whitney's isomorphism theorem (Theorem \ref{whitney}) and establish a correspondence between a root graph $G$ and the line graph $\LL(G)$ (given that $G \ncong K_{1,3}$), and apply Theorem \ref{thm:unique_approximation} (Appendix \ref{apd:univapprox}) to link the isomorphism testing to universal function approximation. Excluding claw graphs, any pair of non-isomorphic root graphs would have non-isomorphic line graphs. 
 
 Furthermore, we can extend Theorem \ref{thm:equivalence} to Corollary \ref{cor:n_step_linetransform} by the claw-free property of line graphs (Corollary \ref{thm:claw-free}). As $\LL(G)$ could not be $K_{1,3}$, it satisfies the assumption for Theorem \ref{thm:equivalence} to extend the equivalence of universal function approximation between $\LL^{(2)}(G)$ and $\LL(G).$ Thus, by induction, the theorem can be extended to $n$-step line graph transformations for an arbitrary number of $n$.

\begin{corollary}
\label{cor:n_step_linetransform}
    Let \(\mathcal{G}\) be a set of connected non-claw graphs and \(\mathcal{C}\) be a collection of functions. If $\forall G_1,G_2\in\mathcal{G}$, \(G_1\ncong G_2\), and \(\forall n \in \mathbb{N}\), \(\exists h\in \mathcal{C}\) such that \(h(\LL^{(n)}(G_1))\neq h(\LL^{(n)}(G_2))\) Then, \(\mathcal{C}\) can universally approximate any permutation-invariant function $f:\mathcal{G}\rightarrow \mathbb{R}$.
\end{corollary}

\subsection{Implication on challenging graphs}
The constraints on line graphs as described by \citet{beineke1970characterizations} offer a silver lining: despite the potential increase in graph size, the transformed graph may be structurally simpler. In this section, we introduce two challenging graphs in WL tests, namely CFI graphs and strongly regular graphs shown in Figure \ref{fig:srg_cfi_example}. We demonstrate how line graph transformation excludes the existence of such counter-examples for the isomorphism testing.
\begin{figure}
    \centering
    \subfigure[CFI]{  
    \begin{minipage}[t]{0.47\textwidth}
    \centering   \begin{tikzpicture}[scale=1.8]
      \draw
        (0.724, 0.221) node[draw, circle, fill=black, minimum size=4pt, inner sep=0pt] (0){}
        (0.407, 0.243) node[draw, circle, fill=black, minimum size=4pt, inner sep=0pt] (1){}
        (0.51, 0.249) node[draw, circle, fill=black, minimum size=4pt, inner sep=0pt] (2){}
        (0.922, 0.016) node[draw, circle, fill=black, minimum size=4pt, inner sep=0pt] (3){}
        (0.285, 0.429) node[draw, circle, fill=black, minimum size=4pt, inner sep=0pt] (4){}
        (0.419, 0.071) node[draw, circle, fill=black, minimum size=4pt, inner sep=0pt] (5){}
        (0.337, 0.499) node[draw, circle, fill=black, minimum size=4pt, inner sep=0pt] (6){}
        (0.506, 0.032) node[draw, circle, fill=black, minimum size=4pt, inner sep=0pt] (7){}
        (0.753, 0.405) node[draw, circle, fill=black, minimum size=4pt, inner sep=0pt] (8){}
        (0.523, 0.418) node[draw, circle, fill=black, minimum size=4pt, inner sep=0pt] (9){}
        (0.536, 0.331) node[draw, circle, fill=black, minimum size=4pt, inner sep=0pt] (10){}
        (0.966, 0.23) node[draw, circle, fill=black, minimum size=4pt, inner sep=0pt] (11){}
        (0.046, 0.538) node[draw, circle, fill=black, minimum size=4pt, inner sep=0pt] (12){}
        (0.081, 0.602) node[draw, circle, fill=black, minimum size=4pt, inner sep=0pt] (13){}
        (0.457, -0.243) node[draw, circle, fill=black, minimum size=4pt, inner sep=0pt] (14){}
        (0.361, -0.192) node[draw, circle, fill=black, minimum size=4pt, inner sep=0pt] (15){}
        (-0.722, -0.096) node[draw, circle, fill=black, minimum size=4pt, inner sep=0pt] (16){}
        (-0.779, 0.235) node[draw, circle, fill=black, minimum size=4pt, inner sep=0pt] (17){}
        (-0.547, -0.361) node[draw, circle, fill=black, minimum size=4pt, inner sep=0pt] (18){}
        (-0.634, 0.543) node[draw, circle, fill=black, minimum size=4pt, inner sep=0pt] (19){}
        (-0.788, -0.273) node[draw, circle, fill=black, minimum size=4pt, inner sep=0pt] (20){}
        (-0.811, 0.066) node[draw, circle, fill=black, minimum size=4pt, inner sep=0pt] (21){}
        (-0.586, -0.513) node[draw, circle, fill=black, minimum size=4pt, inner sep=0pt] (22){}
        (-0.637, 0.354) node[draw, circle, fill=black, minimum size=4pt, inner sep=0pt] (23){}
        (-0.358, 0.692) node[draw, circle, fill=black, minimum size=4pt, inner sep=0pt] (24){}
        (-0.377, 0.535) node[draw, circle, fill=black, minimum size=4pt, inner sep=0pt] (25){}
        (-0.584, 0.079) node[draw, circle, fill=black, minimum size=4pt, inner sep=0pt] (26){}
        (-0.688, 0.411) node[draw, circle, fill=black, minimum size=4pt, inner sep=0pt] (27){}
        (-0.424, -0.172) node[draw, circle, fill=black, minimum size=4pt, inner sep=0pt] (28){}
        (-0.568, 0.728) node[draw, circle, fill=black, minimum size=4pt, inner sep=0pt] (29){}
        (-0.795, -0.47) node[draw, circle, fill=black, minimum size=4pt, inner sep=0pt] (30){}
        (-0.791, -0.131) node[draw, circle, fill=black, minimum size=4pt, inner sep=0pt] (31){}
        (-0.596, -0.69) node[draw, circle, fill=black, minimum size=4pt, inner sep=0pt] (32){}
        (-0.572, 0.14) node[draw, circle, fill=black, minimum size=4pt, inner sep=0pt] (33){}
        (-0.301, 0.861) node[draw, circle, fill=black, minimum size=4pt, inner sep=0pt] (34){}
        (-0.331, 0.315) node[draw, circle, fill=black, minimum size=4pt, inner sep=0pt] (35){}
        (0.957, -0.291) node[draw, circle, fill=black, minimum size=4pt, inner sep=0pt] (36){}
        (0.813, -0.566) node[draw, circle, fill=black, minimum size=4pt, inner sep=0pt] (37){}
        (1.0, -0.078) node[draw, circle, fill=black, minimum size=4pt, inner sep=0pt] (38){}
        (0.856, -0.352) node[draw, circle, fill=black, minimum size=4pt, inner sep=0pt] (39){}
        (0.551, -0.648) node[draw, circle, fill=black, minimum size=4pt, inner sep=0pt] (40){}
        (0.598, -0.472) node[draw, circle, fill=black, minimum size=4pt, inner sep=0pt] (41){}
        (-0.294, -0.341) node[draw, circle, fill=black, minimum size=4pt, inner sep=0pt] (42){}
        (-0.396, -0.593) node[draw, circle, fill=black, minimum size=4pt, inner sep=0pt] (43){}
        (-0.149, -0.59) node[draw, circle, fill=black, minimum size=4pt, inner sep=0pt] (44){}
        (-0.149, -0.517) node[draw, circle, fill=black, minimum size=4pt, inner sep=0pt] (45){}
        (-0.387, -0.716) node[draw, circle, fill=black, minimum size=4pt, inner sep=0pt] (46){}
        (-0.371, -0.417) node[draw, circle, fill=black, minimum size=4pt, inner sep=0pt] (47){}
        (0.115, -0.503) node[draw, circle, fill=black, minimum size=4pt, inner sep=0pt] (48){}
        (0.125, -0.587) node[draw, circle, fill=black, minimum size=4pt, inner sep=0pt] (49){}
        (-0.197, 0.523) node[draw, circle, fill=black, minimum size=4pt, inner sep=0pt] (50){}
        (-0.155, 0.698) node[draw, circle, fill=black, minimum size=4pt, inner sep=0pt] (51){}
        (-0.146, 0.425) node[draw, circle, fill=black, minimum size=4pt, inner sep=0pt] (52){}
        (-0.096, 0.807) node[draw, circle, fill=black, minimum size=4pt, inner sep=0pt] (53){}
        (0.368, -0.584) node[draw, circle, fill=black, minimum size=4pt, inner sep=0pt] (54){}
        (0.393, -0.422) node[draw, circle, fill=black, minimum size=4pt, inner sep=0pt] (55){}
        (0.298, -0.38) node[draw, circle, fill=black, minimum size=4pt, inner sep=0pt] (56){}
        (0.324, -0.497) node[draw, circle, fill=black, minimum size=4pt, inner sep=0pt] (57){};
      \begin{scope}[-]
        \draw (0) to (1);
        \draw (0) to (2);
        \draw (0) to (3);
        \draw (1) to (4);
        \draw (1) to (5);
        \draw (2) to (6);
        \draw (2) to (7);
        \draw (3) to (36);
        \draw (4) to (10);
        \draw (4) to (12);
        \draw (5) to (9);
        \draw (5) to (15);
        \draw (6) to (9);
        \draw (6) to (13);
        \draw (7) to (10);
        \draw (7) to (14);
        \draw (8) to (9);
        \draw (8) to (10);
        \draw (8) to (11);
        \draw (11) to (38);
        \draw (12) to (50);
        \draw (12) to (51);
        \draw (13) to (52);
        \draw (13) to (53);
        \draw (14) to (54);
        \draw (14) to (55);
        \draw (15) to (56);
        \draw (15) to (57);
        \draw (16) to (17);
        \draw (16) to (18);
        \draw (17) to (19);
        \draw (18) to (42);
        \draw (18) to (43);
        \draw (19) to (24);
        \draw (20) to (21);
        \draw (20) to (22);
        \draw (21) to (23);
        \draw (22) to (46);
        \draw (22) to (47);
        \draw (23) to (25);
        \draw (24) to (50);
        \draw (24) to (53);
        \draw (25) to (51);
        \draw (25) to (52);
        \draw (26) to (27);
        \draw (26) to (28);
        \draw (27) to (29);
        \draw (28) to (42);
        \draw (28) to (47);
        \draw (29) to (34);
        \draw (30) to (31);
        \draw (30) to (32);
        \draw (31) to (33);
        \draw (32) to (43);
        \draw (32) to (46);
        \draw (33) to (35);
        \draw (34) to (51);
        \draw (34) to (53);
        \draw (35) to (50);
        \draw (35) to (52);
        \draw (36) to (37);
        \draw (37) to (40);
        \draw (38) to (39);
        \draw (39) to (41);
        \draw (40) to (54);
        \draw (40) to (57);
        \draw (41) to (55);
        \draw (41) to (56);
        \draw (42) to (44);
        \draw (43) to (45);
        \draw (44) to (46);
        \draw (44) to (49);
        \draw (45) to (47);
        \draw (45) to (48);
        \draw (48) to (55);
        \draw (48) to (57);
        \draw (49) to (54);
        \draw (49) to (56);
      \end{scope}
    \end{tikzpicture}
    \begin{tikzpicture}[scale=1.8]
      \draw
        (0.727, 0.247) node[draw, circle, fill=black, minimum size=4pt, inner sep=0pt] (0){}
        (0.408, 0.245) node[draw, circle, fill=black, minimum size=4pt, inner sep=0pt] (1){}
        (0.538, 0.296) node[draw, circle, fill=black, minimum size=4pt, inner sep=0pt] (2){}
        (0.93, 0.05) node[draw, circle, fill=black, minimum size=4pt, inner sep=0pt] (3){}
        (0.255, 0.431) node[draw, circle, fill=black, minimum size=4pt, inner sep=0pt] (4){}
        (0.421, 0.077) node[draw, circle, fill=black, minimum size=4pt, inner sep=0pt] (5){}
        (0.325, 0.519) node[draw, circle, fill=black, minimum size=4pt, inner sep=0pt] (6){}
        (0.501, 0.067) node[draw, circle, fill=black, minimum size=4pt, inner sep=0pt] (7){}
        (0.727, 0.437) node[draw, circle, fill=black, minimum size=4pt, inner sep=0pt] (8){}
        (0.513, 0.434) node[draw, circle, fill=black, minimum size=4pt, inner sep=0pt] (9){}
        (0.474, 0.357) node[draw, circle, fill=black, minimum size=4pt, inner sep=0pt] (10){}
        (0.952, 0.274) node[draw, circle, fill=black, minimum size=4pt, inner sep=0pt] (11){}
        (0.018, 0.539) node[draw, circle, fill=black, minimum size=4pt, inner sep=0pt] (12){}
        (0.061, 0.607) node[draw, circle, fill=black, minimum size=4pt, inner sep=0pt] (13){}
        (0.467, -0.207) node[draw, circle, fill=black, minimum size=4pt, inner sep=0pt] (14){}
        (0.379, -0.195) node[draw, circle, fill=black, minimum size=4pt, inner sep=0pt] (15){}
        (-0.72, -0.126) node[draw, circle, fill=black, minimum size=4pt, inner sep=0pt] (16){}
        (-0.791, 0.204) node[draw, circle, fill=black, minimum size=4pt, inner sep=0pt] (17){}
        (-0.533, -0.383) node[draw, circle, fill=black, minimum size=4pt, inner sep=0pt] (18){}
        (-0.658, 0.517) node[draw, circle, fill=black, minimum size=4pt, inner sep=0pt] (19){}
        (-0.778, -0.307) node[draw, circle, fill=black, minimum size=4pt, inner sep=0pt] (20){}
        (-0.816, 0.032) node[draw, circle, fill=black, minimum size=4pt, inner sep=0pt] (21){}
        (-0.566, -0.538) node[draw, circle, fill=black, minimum size=4pt, inner sep=0pt] (22){}
        (-0.653, 0.327) node[draw, circle, fill=black, minimum size=4pt, inner sep=0pt] (23){}
        (-0.388, 0.678) node[draw, circle, fill=black, minimum size=4pt, inner sep=0pt] (24){}
        (-0.4, 0.52) node[draw, circle, fill=black, minimum size=4pt, inner sep=0pt] (25){}
        (-0.589, 0.055) node[draw, circle, fill=black, minimum size=4pt, inner sep=0pt] (26){}
        (-0.706, 0.384) node[draw, circle, fill=black, minimum size=4pt, inner sep=0pt] (27){}
        (-0.417, -0.19) node[draw, circle, fill=black, minimum size=4pt, inner sep=0pt] (28){}
        (-0.599, 0.706) node[draw, circle, fill=black, minimum size=4pt, inner sep=0pt] (29){}
        (-0.776, -0.504) node[draw, circle, fill=black, minimum size=4pt, inner sep=0pt] (30){}
        (-0.786, -0.165) node[draw, circle, fill=black, minimum size=4pt, inner sep=0pt] (31){}
        (-0.567, -0.717) node[draw, circle, fill=black, minimum size=4pt, inner sep=0pt] (32){}
        (-0.578, 0.116) node[draw, circle, fill=black, minimum size=4pt, inner sep=0pt] (33){}
        (-0.337, 0.85) node[draw, circle, fill=black, minimum size=4pt, inner sep=0pt] (34){}
        (-0.344, 0.301) node[draw, circle, fill=black, minimum size=4pt, inner sep=0pt] (35){}
        (1.0, -0.035) node[draw, circle, fill=black, minimum size=4pt, inner sep=0pt] (36){}
        (0.866, -0.314) node[draw, circle, fill=black, minimum size=4pt, inner sep=0pt] (37){}
        (0.975, -0.257) node[draw, circle, fill=black, minimum size=4pt, inner sep=0pt] (38){}
        (0.84, -0.536) node[draw, circle, fill=black, minimum size=4pt, inner sep=0pt] (39){}
        (0.611, -0.443) node[draw, circle, fill=black, minimum size=4pt, inner sep=0pt] (40){}
        (0.582, -0.632) node[draw, circle, fill=black, minimum size=4pt, inner sep=0pt] (41){}
        (-0.278, -0.354) node[draw, circle, fill=black, minimum size=4pt, inner sep=0pt] (42){}
        (-0.374, -0.61) node[draw, circle, fill=black, minimum size=4pt, inner sep=0pt] (43){}
        (-0.122, -0.598) node[draw, circle, fill=black, minimum size=4pt, inner sep=0pt] (44){}
        (-0.131, -0.523) node[draw, circle, fill=black, minimum size=4pt, inner sep=0pt] (45){}
        (-0.356, -0.733) node[draw, circle, fill=black, minimum size=4pt, inner sep=0pt] (46){}
        (-0.357, -0.431) node[draw, circle, fill=black, minimum size=4pt, inner sep=0pt] (47){}
        (0.133, -0.497) node[draw, circle, fill=black, minimum size=4pt, inner sep=0pt] (48){}
        (0.152, -0.585) node[draw, circle, fill=black, minimum size=4pt, inner sep=0pt] (49){}
        (-0.223, 0.514) node[draw, circle, fill=black, minimum size=4pt, inner sep=0pt] (50){}
        (-0.187, 0.692) node[draw, circle, fill=black, minimum size=4pt, inner sep=0pt] (51){}
        (-0.161, 0.421) node[draw, circle, fill=black, minimum size=4pt, inner sep=0pt] (52){}
        (-0.127, 0.803) node[draw, circle, fill=black, minimum size=4pt, inner sep=0pt] (53){}
        (0.39, -0.412) node[draw, circle, fill=black, minimum size=4pt, inner sep=0pt] (54){}
        (0.406, -0.498) node[draw, circle, fill=black, minimum size=4pt, inner sep=0pt] (55){}
        (0.355, -0.56) node[draw, circle, fill=black, minimum size=4pt, inner sep=0pt] (56){}
        (0.309, -0.351) node[draw, circle, fill=black, minimum size=4pt, inner sep=0pt] (57){};
      \begin{scope}[-]
        \draw (0) to (1);
        \draw (0) to (2);
        \draw (0) to (3);
        \draw (1) to (4);
        \draw (1) to (5);
        \draw (2) to (6);
        \draw (2) to (7);
        \draw (3) to (38);
        \draw (4) to (10);
        \draw (4) to (12);
        \draw (5) to (9);
        \draw (5) to (15);
        \draw (6) to (9);
        \draw (6) to (13);
        \draw (7) to (10);
        \draw (7) to (14);
        \draw (8) to (9);
        \draw (8) to (10);
        \draw (8) to (11);
        \draw (11) to (36);
        \draw (12) to (50);
        \draw (12) to (51);
        \draw (13) to (52);
        \draw (13) to (53);
        \draw (14) to (54);
        \draw (14) to (55);
        \draw (15) to (56);
        \draw (15) to (57);
        \draw (16) to (17);
        \draw (16) to (18);
        \draw (17) to (19);
        \draw (18) to (42);
        \draw (18) to (43);
        \draw (19) to (24);
        \draw (20) to (21);
        \draw (20) to (22);
        \draw (21) to (23);
        \draw (22) to (46);
        \draw (22) to (47);
        \draw (23) to (25);
        \draw (24) to (50);
        \draw (24) to (53);
        \draw (25) to (51);
        \draw (25) to (52);
        \draw (26) to (27);
        \draw (26) to (28);
        \draw (27) to (29);
        \draw (28) to (42);
        \draw (28) to (47);
        \draw (29) to (34);
        \draw (30) to (31);
        \draw (30) to (32);
        \draw (31) to (33);
        \draw (32) to (43);
        \draw (32) to (46);
        \draw (33) to (35);
        \draw (34) to (51);
        \draw (34) to (53);
        \draw (35) to (50);
        \draw (35) to (52);
        \draw (36) to (37);
        \draw (37) to (40);
        \draw (38) to (39);
        \draw (39) to (41);
        \draw (40) to (54);
        \draw (40) to (57);
        \draw (41) to (55);
        \draw (41) to (56);
        \draw (42) to (44);
        \draw (43) to (45);
        \draw (44) to (46);
        \draw (44) to (49);
        \draw (45) to (47);
        \draw (45) to (48);
        \draw (48) to (55);
        \draw (48) to (57);
        \draw (49) to (54);
        \draw (49) to (56);
      \end{scope}
    \end{tikzpicture}
    \end{minipage}
}
    \subfigure[Strongly Regular]{  
    \begin{minipage}[t]{0.47\textwidth}
    \centering 
\begin{tikzpicture}[scale = 1.5]
      \draw
        (0.76, 0.027) node[draw, circle, fill=black, minimum size=4pt, inner sep=0pt] (0){}
        (1.0, 0.352) node[draw, circle, fill=black, minimum size=4pt, inner sep=0pt] (1){}
        (0.708, 0.789) node[draw, circle, fill=black, minimum size=4pt, inner sep=0pt] (2){}
        (0.316, 0.692) node[draw, circle, fill=black, minimum size=4pt, inner sep=0pt] (3){}
        (-0.027, 0.76) node[draw, circle, fill=black, minimum size=4pt, inner sep=0pt] (4){}
        (-0.352, 1.0) node[draw, circle, fill=black, minimum size=4pt, inner sep=0pt] (5){}
        (-0.789, 0.708) node[draw, circle, fill=black, minimum size=4pt, inner sep=0pt] (6){}
        (-0.692, 0.316) node[draw, circle, fill=black, minimum size=4pt, inner sep=0pt] (7){}
        (-0.76, -0.027) node[draw, circle, fill=black, minimum size=4pt, inner sep=0pt] (8){}
        (-1.0, -0.352) node[draw, circle, fill=black, minimum size=4pt, inner sep=0pt] (9){}
        (-0.708, -0.789) node[draw, circle, fill=black, minimum size=4pt, inner sep=0pt] (10){}
        (-0.316, -0.692) node[draw, circle, fill=black, minimum size=4pt, inner sep=0pt] (11){}
        (0.027, -0.76) node[draw, circle, fill=black, minimum size=4pt, inner sep=0pt] (12){}
        (0.352, -1.0) node[draw, circle, fill=black, minimum size=4pt, inner sep=0pt] (13){}
        (0.789, -0.708) node[draw, circle, fill=black, minimum size=4pt, inner sep=0pt] (14){}
        (0.692, -0.316) node[draw, circle, fill=black, minimum size=4pt, inner sep=0pt] (15){};
      \begin{scope}[-]
        \draw (0) to (1);
        \draw (0) to (2);
        \draw (0) to (3);
        \draw (0) to (4);
        \draw (0) to (8);
        \draw (0) to (12);
        \draw (1) to (2);
        \draw (1) to (3);
        \draw (1) to (5);
        \draw (1) to (9);
        \draw (1) to (13);
        \draw (2) to (3);
        \draw (2) to (6);
        \draw (2) to (10);
        \draw (2) to (14);
        \draw (3) to (7);
        \draw (3) to (11);
        \draw (3) to (15);
        \draw (4) to (5);
        \draw (4) to (6);
        \draw (4) to (7);
        \draw (4) to (8);
        \draw (4) to (12);
        \draw (5) to (6);
        \draw (5) to (7);
        \draw (5) to (9);
        \draw (5) to (13);
        \draw (6) to (7);
        \draw (6) to (10);
        \draw (6) to (14);
        \draw (7) to (11);
        \draw (7) to (15);
        \draw (8) to (9);
        \draw (8) to (10);
        \draw (8) to (11);
        \draw (8) to (12);
        \draw (9) to (10);
        \draw (9) to (11);
        \draw (9) to (13);
        \draw (10) to (11);
        \draw (10) to (14);
        \draw (11) to (15);
        \draw (12) to (13);
        \draw (12) to (14);
        \draw (12) to (15);
        \draw (13) to (14);
        \draw (13) to (15);
        \draw (14) to (15);
      \end{scope}
    \end{tikzpicture}
\begin{tikzpicture}[scale=1.5]
      \draw
        (1.0, 0.282) node[draw, circle, fill=black, minimum size=4pt, inner sep=0pt] (0){}
        (0.602, -0.02) node[draw, circle, fill=black, minimum size=4pt, inner sep=0pt] (1){}
        (0.211, 0.564) node[draw, circle, fill=black, minimum size=4pt, inner sep=0pt] (2){}
        (0.643, 0.816) node[draw, circle, fill=black, minimum size=4pt, inner sep=0pt] (3){}
        (-0.282, 1.0) node[draw, circle, fill=black, minimum size=4pt, inner sep=0pt] (4){}
        (0.02, 0.602) node[draw, circle, fill=black, minimum size=4pt, inner sep=0pt] (5){}
        (-0.564, 0.211) node[draw, circle, fill=black, minimum size=4pt, inner sep=0pt] (6){}
        (-0.816, 0.643) node[draw, circle, fill=black, minimum size=4pt, inner sep=0pt] (7){}
        (-1.0, -0.282) node[draw, circle, fill=black, minimum size=4pt, inner sep=0pt] (8){}
        (-0.602, 0.02) node[draw, circle, fill=black, minimum size=4pt, inner sep=0pt] (9){}
        (-0.211, -0.564) node [draw, circle, fill=black, minimum size=4pt, inner sep=0pt](10){}
        (-0.643, -0.816) node[draw, circle, fill=black, minimum size=4pt, inner sep=0pt] (11){}
        (0.282, -1.0) node[draw, circle, fill=black, minimum size=4pt, inner sep=0pt] (12){}
        (-0.02, -0.602) node[draw, circle, fill=black, minimum size=4pt, inner sep=0pt] (13){}
        (0.564, -0.211) node[draw, circle, fill=black, minimum size=4pt, inner sep=0pt] (14){}
        (0.816, -0.643) node[draw, circle, fill=black, minimum size=4pt, inner sep=0pt] (15){};
      \begin{scope}[-]
        \draw (0) to (1);
        \draw (0) to (3);
        \draw (0) to (4);
        \draw (0) to (5);
        \draw (0) to (12);
        \draw (0) to (15);
        \draw (1) to (2);
        \draw (1) to (5);
        \draw (1) to (6);
        \draw (1) to (12);
        \draw (1) to (13);
        \draw (2) to (3);
        \draw (2) to (6);
        \draw (2) to (7);
        \draw (2) to (13);
        \draw (2) to (14);
        \draw (3) to (4);
        \draw (3) to (7);
        \draw (3) to (14);
        \draw (3) to (15);
        \draw (4) to (5);
        \draw (4) to (7);
        \draw (4) to (8);
        \draw (4) to (9);
        \draw (5) to (6);
        \draw (5) to (9);
        \draw (5) to (10);
        \draw (6) to (7);
        \draw (6) to (10);
        \draw (6) to (11);
        \draw (7) to (8);
        \draw (7) to (11);
        \draw (8) to (9);
        \draw (8) to (11);
        \draw (8) to (12);
        \draw (8) to (13);
        \draw (9) to (10);
        \draw (9) to (13);
        \draw (9) to (14);
        \draw (10) to (11);
        \draw (10) to (14);
        \draw (10) to (15);
        \draw (11) to (12);
        \draw (11) to (15);
        \draw (12) to (13);
        \draw (12) to (15);
        \draw (13) to (14);
        \draw (14) to (15);
      \end{scope}
    \end{tikzpicture}
    \end{minipage}}
    \caption{Example of a pair of CFI graphs and a pair of strongly regular graphs.}
    \label{fig:srg_cfi_example}
\end{figure}
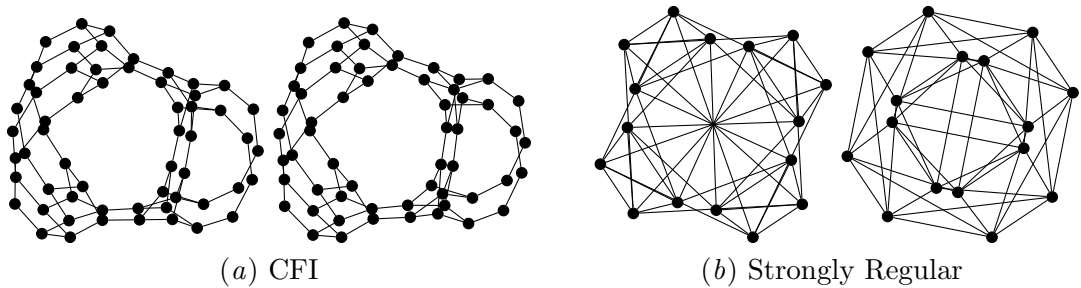

\subsubsection{Cai-Fürer-Immerman Graphs}

The Cai-Fürer-Immerman (CFI) graphs~\citep{cai1992optimal} are famously constructed to demonstrate the difficulty of the graph isomorphism problem, especially regarding WL tests, shown in Figure \ref{fig:srg_cfi_example}, (a). In fact, \citet{cai1992optimal} has shown that there exists a CFI construction that is indistinguishable under the \(k\)-WL test for any fixed \(k\). The CFI graph involves the construction of $X_k = (V_k, E_k)$ with
$V_k = A_k \cup B_k \cup M_k$ where 
\begin{align*}
A_k &= \{a_i \mid 1 \leq i \leq k\}, B_k = \{b_i \mid 1 \leq i \leq k\} \\
M_k &= \{m_S \mid S \subseteq \{1, \ldots, k\}, |S| \text{ is even}\} \\
E_k &= \{(m_S, a_i) \mid i \in S\} \cup \{(m_S, b_i) \mid i \notin S\}.
\end{align*}
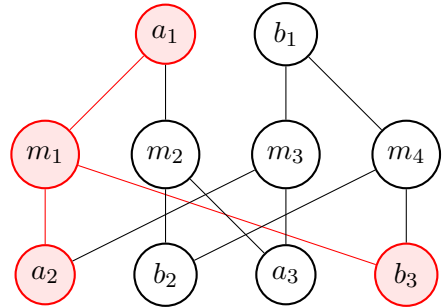
\begin{wrapfigure}{r}{6.5cm}
\centering
\begin{tikzpicture}[node distance=1.6cm,
                    main node/.style={circle, draw=black, thick, minimum size=7mm},
                    red node/.style={circle, draw=red, thick, fill=red!10, minimum size=7mm},
                    ] % Red edge style
  \tikzset{
    edge/.style = {draw},
    red edge/.style = {draw=red}
  }
  \begin{scope}[scale=0.8]

  % Nodes with LaTeX labels, specifying red nodes
  \node[red node] (m1) {${m_1}$};
  \node[main node] (m2) [right of=m1] {${m_2}$};
  \node[main node] (m3) [right of=m2] {${m_3}$};
  \node[main node] (m4) [right of=m3] {${m_4}$};
  
  \node[red node] (a1) [above of=m2] {${a_1}$};
  \node[red node] (a2) [below of=m1] {${a_2}$};
  \node[main node] (a3) [below of=m3] {${a_3}$};
  
  \node[main node] (b1) [above of=m3] {${b_1}$};
  \node[main node] (b2) [below of=m2] {${b_2}$};
  \node[red node] (b3) [below of=m4] {${b_3}$};
  
  % Edges, specifying red edges
  \draw[red edge] (m1) -- (a2);
  \draw[red edge] (m1) -- (a1);
  \draw[edge] (m2) -- (a1);
  \draw[red edge] (m1) -- (b3);
  \draw[edge] (m2) -- (b2);
  \draw[edge] (m2) -- (a3);
  \draw[edge] (m3) -- (a2);
  \draw[edge] (m3) -- (a3);
  \draw[edge] (m3) -- (b1);
  \draw[edge] (m4) -- (b1);
  \draw[edge] (m4) -- (b2);
  \draw[edge] (m4) -- (b3);
  \end{scope}
\end{tikzpicture}
\caption{An CFI graph substructure $X_3$ and its induced subgraph $K_{1,3}$ in red.}
\label{fig:X3}
\vspace{-4em}
\end{wrapfigure}

However, we observe that applying line graph transformation on CFI graphs can provably remove this construction, since the CFI graphs $X_k$ with $k \geq 3$ has to include the claw graph $K_{1,3}$ as an induced subgraph, thus ruling them out from being a line graph due to Corollary $\ref{thm:claw-free}$. Figure \ref{fig:X3} presents an example of $K_{1,3}$ as a subgraph of $X_3$. 

\begin{theorem}
    \label{thm:cfi}
    Line graphs do not include CFI graphs constructed with $X_k$ for $k\geq3$.
\end{theorem}
 Theorem \ref{thm:cfi} demonstrates that after line graph transformation, complex CFI graphs would be non-existent. This prevents graph isomorphism tests and GNNs from dealing with these difficult instances, potentially assisting these models to distinguish challenging graph pairs and thus allowing GNNs to better learn on these graphs with the help of Theorem \ref{thm:equivalence}.

\subsubsection{Strongly Regular Graphs}

Another well-known example of the limitation of the WL test is its challenge to differentiate between strongly regular graphs (See Appendix \ref{apd:regular} for details). To distinguish non-isomorphic strongly regular graphs, graph isomorphism tests or GNNs that are 4-WL expressive are required~\citep{bouritsas2020srg3wl}. We present an example strongly regular graph pair that is indistinguishable by 3-WL in Figure \ref{fig:srg_cfi_example}, (b). 

In this section, we demonstrate that the application of line graphs disrupts the strong regularity of these graphs, with the exception of three specific cases. 
We start with the following lemma to showcase the simple cases where the root graph contains $C_3$ as a subgraph (i.e. triangle-containing) and is not complete.

\begin{theorem}
\label{thm:tri-line-graph-regularity}
Let $G$ be a connected, strongly regular graph which is not the complete graph $K_n$ for $n \geq 2$, and $G$ contains $C_3$ as a subgraph. Then $\LL(G)$ is not strongly regular.
\end{theorem}

Furthermore, we show that 
with a finite number of exceptions, the application of most second-order line graph transformations effectively disrupts the strong regularity of graphs. The resulting graphs therefore may no longer require 4-WL tests or equivalent GNN architectures for graph isomorphism testing.

\begin{theorem}
Let $G$ be a connected graph that is not the cycle graphs $C_3$, $C_4$, or $C_5$. Then, applying at most two line graph transformations to $G$ yields a graph that is not strongly regular.
\label{main_thm}
\end{theorem}

\subsubsection{Regular graphs}
It is worth mentioning that the line graph of a regular graph remains regular \citep{RAMANE2005reg}. 
Thus, there exists a pair of graphs, \( G_1 \) and \( G_2 \), such that are at least 2-WL indistinguishable between \( \LL^{(n)}(G_1) \) and \( \LL^{(n)}(G_2) \) for any \( n \geq 0 \).

\subsection{Extension on disconnected graphs}
Notice that all of the aforementioned theorems assume that the root graph is connected. In fact, we can show that given a disconnected graph where each of its connected components is not isomorphic to $K_{1,3}$ or short paths, Corollary \ref{cor:n_step_linetransform} still holds.

\begin{corollary}
    \label{cor:disconnected_equiv}
    Let \(\mathcal{G}\) be a set of graphs such that $\forall G\in\mathcal{G}$, for a fixed $n$, $G$ does not have a component isomorphic to $K_{1,3}$ or $P_k$, $k\leq n$ and Let \(\mathcal{C}\) be a collection of functions. If $\forall G_1,G_2\in\mathcal{G}$ such that \(G_1\ncong G_2\), \(\exists h\in \mathcal{C}\) such that \(h(\LL^{(n)}(G_1))\neq h(\LL^{(n)}(G_2))\). Then, \(\mathcal{C}\) can universally approximate any permutation-invariant function $f:\mathcal{G}\rightarrow \mathbb{R}$.
\end{corollary}

Similarly, assuming a disconnected graph where each of its connected components is not isomorphic to $C_3$, Corollary \ref{main_thm} holds for strongly regular graphs.  

\begin{corollary}
    \label{cor:disconnected_srg}
    Let $G$ be a graph that is not the cycle graphs $C_3$, $C_4$, $C_5$, or a disjoint group of $C_3$. Then, applying at most two line graph transformations to $G$ yields a graph that is not strongly regular.
\end{corollary}

\section{Empirical Evidence}

We validate our theoretical findings through graph classification experiments, designed to answer the following question: \emph{What are the effects of line graph transformation on regular graphs and CFI graphs?}

\subsection{Experimental Setup}
\paragraph{Dataset} We utilize the BREC dataset \citep{wang2024brec} to examine the effects of line graph transformation on both CFI graphs and regular graphs. The dataset consists of 100 pairs of CFI graphs and 120 pairs of regular graphs. Among the regular graphs, 50 pairs are simple regular graphs, and 70 pairs are strongly regular graphs. These graph pairs are pre-processed into two distinct groups: with and without line graph transformation.

\paragraph{Experimental design} In our experimental setup, we conduct graph isomorphism tests on 100 pairs of CFI graphs and 140 pairs of regular graphs, both before and after applying the line graph transformation, using the 3-WL test. Additionally, we apply the 4-WL on the root graph pairs and compare its expressive power on root graphs with that of 3-WL on their corresponding line graphs. To further demonstrate, we also evaluate a GNN model, Provably Powerful Graph Networks (PPGN)~\citep{maron2020provablypowerfulgraphnetworks}, whose expressive power is bounded by 3-WL (See Appendix \ref{apd:second} for more details).

\subsection{Results}

\paragraph{Simple and strongly regular graphs} We present the experiment results in Table \ref{tab:experiment-1}. First, we note that the 3-WL or more expressive algorithms successfully distinguish all pairs of simple regular graphs in the dataset. Regarding strongly regular graphs which are distinguishable by 4-WL but not 3-WL, we observe that 3-WL can successfully differentiate all pairs of strongly regular graphs after applying the line graph transformation, which was previously indistinguishable. This is consistent with our theoretical analysis, suggesting that the strong regularity of graphs may be disrupted after line graph transformation.

\paragraph{CFI graphs} It is noteworthy that the number of pairs of CFI graphs successfully distinguished by the 3-WL algorithm does not increase after a single line graph transformation. Although our theoretical analysis indicates transformed graph does not belong to the CFI graph, the resulting line graph still poses significant challenges for the WL algorithm.

\paragraph{WL and PPGN} We provide additional experimental insights into the comparison between WL algorithms and a graph neural network instance PPGN. In the regular graph experiments, PPGN aligns its performance with 3-WL before and after the line graph transformation. 
However, when applied to CFI graphs, PPGN performs slightly worse on the line graphs compared to the root graphs. As noted in~\citet{wang2024brec}, PPGN struggles to match the performance of 3-WL on CFI graphs due to their large radii. The line graph transformation further increases the size of the CFI graphs, magnifying this issue and leading to decreased performance in PPGN experiments on the line graphs.

\begin{table*}
\vspace{-0.15in}
\caption{The accuracies of distinguishing graph pairs in the BREC regular and CFI graphs.}
\vspace{-0.15in}
\label{tab:experiment-1}
\begin{center}
% \begin{sc}
\resizebox{1.0\textwidth}{!}{
\begin{tabular}{cccccccccccc}
\toprule
~ & ~ & \multicolumn{2}{c}{Simple Regular (50)} & \multicolumn{2}{c}{Strongly Regular (70)}  & \multicolumn{2}{c}{CFI (100)} & \\
\cmidrule(r{0.5em}){3-4} \cmidrule(l{0.5em}){5-6}\cmidrule(l{0.5em}){7-8}\cmidrule(l{0.5em}){9-10}\cmidrule(l{0.5em}){11-12}
Graph & Model &  Number & Accuracy & Number & Accuracy    & Number & Accuracy  \\
\midrule
\multirow{3}{*}{$G$}
& 3-WL & 50 & 100\%  & 0 & 0\%  & 60 & 60\%  \\
& 4-WL & 50 & 100\% & 70 & 100\%  & 80 & 80\% \\
& PPGN & 50 & 100\%  & 0 & 0\%  & 22 &22\%  \\
\midrule
\multirow{2}{*}{$L(G)$}
& 3-WL & 50 & 100\% & 70 & 100\%  & 60 & 60\%  \\
& PPGN & 50 & 100\% & 70 & 100\%  & 15 & 15\%  \\

\bottomrule
\end{tabular}
}
\end{center}
\end{table*}

\section{Conclusions}
In this study, we provide a theoretical analysis of applying GNNs on line graphs. With mild assumptions, we show that line graphs are equivalent to the original graphs for isomorphism testing and permutation-invariant function approximation.
In particular, we focus on two challenging graph classes for the isomorphism test, namely CFI graphs and strongly regular graphs, and show that both classes can be excluded or reduced by the application of line graph transformation. Empirically, strongly regular graphs after line graphs are 3-WL distinguishable, whereas the line graphs of CFI graphs remain challenging to the WL tests.

\acks{The authors would like to acknowledge the use of the University of Oxford Advanced Research Computing (ARC) facility in carrying out this work. \url{http://dx.doi.org/10.5281/zenodo.22558}}. In addition, F.Y. would like to thank Dr. Seth Flaxman for his support in attending the conference.

\bibliography{pmlr-sample}
\newpage

\appendix
\section{Weisfeiler-Leman test}
\label{apd:wltest}
The \textit{Weisfeiler-Leman (WL) test} is a graph isomorphism test widely applied in the field of graph theory and graph-based learning. In particular, the \emph{k-dimensional Weisfeiler-Leman algorithm (k-WL)} operates by iteratively assigning colors to $k$-tuples of nodes in a graph. Formally, let \( G = (V, E) \) be a graph, where \( V \) is the set of nodes and \( E \) the set of edges. For a fixed dimension \( k \), the $k$-WL algorithm starts by assigning an initial colouring \( \chi^{(0)}_k \) to every $k$-tuple \( \mathbf{u} = (u_1, \ldots, u_k) \) of nodes. The coloring is defined based on the isomorphism type of the subgraph induced by the nodes in the tuple \( \mathbf{u} \). At each iteration \( r \), the algorithm refines the coloring as follows:

\[
\chi^{(r+1)}_k(\mathbf{u}) = \left( \chi^{(r)}_k(\mathbf{u}), \left\{\chi^{(r)}_k(\mathbf{u}[w/i]) \mid w \in V, i \in [k] \right\} \right),
\]

where \( \mathbf{u}[w/i] \) denotes the tuple obtained by replacing the \( i \)-th node in \( \mathbf{u} \) with node \( w \). The refinement process terminates when further iterations no longer produce a different coloring. 

\section{Universal approximation over permutation-invariant functions}
\label{apd:univapprox}
The link between the universal approximation capabilities of GNNs with respect to graphs and the testing of graph isomorphism was demonstrated by \citet{chen2019universal}. 
In their work, it was shown that the universal approximation of permutation-invariant functions on graphs is equivalent to graph isomorphism testing. This equivalence between universal approximation and graph isomorphism testing stated as Theorem \ref{thm:unique_approximation} serves as the foundation for evaluating the expressiveness of different GNN architectures based on the expressivity of $k$-WL tests.

\begin{theorem}[\citet{chen2019universal}]
\label{thm:unique_approximation}
    A function class is capable of universally approximating permutation-invariant functions on graphs with finite node attributes if and only if it can discriminate non-isomorphic graphs.
\end{theorem}

\section{Beineke's forbidden induced subgraphs}
\label{app:beineke}

\citet{beineke1970characterizations} introduced the following 9 graphs, shown in Figure \ref{fig:b9} to be the forbidden induced subgraphs that characterize line graphs. The theorem states that a graph $H$ is a line graph of a root graph $G$ root if and only if $H$ does not contain any of the nine Beineke's forbidden graphs as an induced subgraph.
\begin{figure}[h]
\centering
    \centering
    \includegraphics[width=1\textwidth]{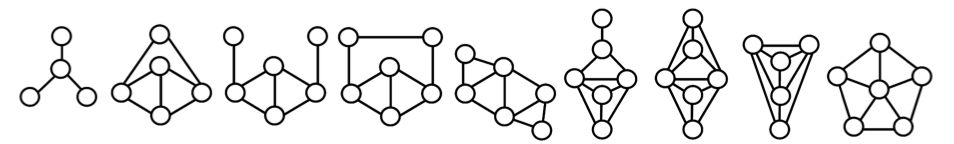}
\caption{Nine Beineke's forbidden induced subgraphs.}
\label{fig:b9}
\end{figure}

\section{Details on regular graphs}
\label{apd:regular}
\subsection{ Regular graphs}
A regular graph is a graph where each node has the same number of adjacent nodes, called the degree of the graph. In other words, a graph is $k$-regular if every node has exactly $k$ neighbors. Examples of regular graphs include cycles and complete graphs.

\subsection{Strongly regular graphs}
\label{srg_def}
 A regular graph of $v$ nodes is defined as strongly regular if there are positive integers \(k\), \(\lambda\), and \(\mu\) satisfying 
\begin{enumerate}
    \item every node is connected to \(k\) other nodes,
    \item each pair of connected nodes shares \(\lambda\) mutual neighbours, and
    \item every pair of nodes that are not directly connected shares \(\mu\) mutual neighbors.
\end{enumerate} 
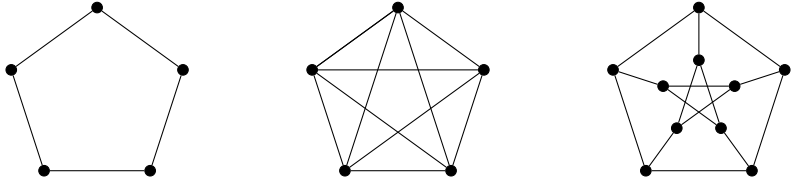
\begin{figure}[h]
\centering
\begin{tikzpicture}
    % First Graph
    \begin{scope}
        \node[draw, circle, fill=black, minimum size=4pt, inner sep=0pt](a1) at (90:1.2) {};
    \node[draw, circle, fill=black, minimum size=4pt, inner sep=0pt](b1) at (162:1.2) {};
    \node[draw, circle, fill=black, minimum size=4pt, inner sep=0pt](c1) at (234:1.2) {};
    \node[draw, circle, fill=black, minimum size=4pt, inner sep=0pt](d1) at (306:1.2) {};
    \node[draw, circle, fill=black, minimum size=4pt, inner sep=0pt](e1) at (18:1.2) {};
        \draw (a1) -- (b1) node[midway, left] {};
        \draw (b1) -- (c1) node[midway, right] {};
        \draw (c1) -- (d1) node[midway, right] {};
        \draw (d1) -- (e1) node[midway, right] {};
        \draw (e1) -- (a1) node[midway, right] {};
        
    \end{scope}
    \begin{scope}[xshift=4cm]
        \node[draw, circle, fill=black, minimum size=4pt, inner sep=0pt](a1) at (90:1.2) {};
    \node[draw, circle, fill=black, minimum size=4pt, inner sep=0pt](b1) at (162:1.2) {};
    \node[draw, circle, fill=black, minimum size=4pt, inner sep=0pt](c1) at (234:1.2) {};
    \node[draw, circle, fill=black, minimum size=4pt, inner sep=0pt](d1) at (306:1.2) {};
    \node[draw, circle, fill=black, minimum size=4pt, inner sep=0pt](e1) at (18:1.2) {};
        \draw (a1) -- (b1) node[midway, left] {};
        \draw (b1) -- (a1) node[midway, right] {};
        \draw (c1) -- (a1) node[midway, right] {};
        \draw (d1) -- (a1) node[midway, right] {};
        \draw (b1) -- (c1) node[midway, right] {};
        \draw (b1) -- (d1) node[midway, right] {};
        \draw (b1) -- (e1) node[midway, right] {};
        \draw (c1) -- (d1) node[midway, right] {};
        \draw (c1) -- (e1) node[midway, right] {};
        \draw (d1) -- (e1) node[midway, right] {};
        \draw (e1) -- (a1) node[midway, right] {};

    \end{scope}
    
    \begin{scope}[xshift=8cm]
        
    % Outer cycle (pentagon)
    \node[draw, circle, fill=black, minimum size=4pt, inner sep=0pt](a1) at (90:1.2) {};
    \node[draw, circle, fill=black, minimum size=4pt, inner sep=0pt](a2) at (162:1.2) {};
    \node[draw, circle, fill=black, minimum size=4pt, inner sep=0pt](a3) at (234:1.2) {};
    \node[draw, circle, fill=black, minimum size=4pt, inner sep=0pt](a4) at (306:1.2) {};
    \node[draw, circle, fill=black, minimum size=4pt, inner sep=0pt](a5) at (18:1.2) {};
    
    % Inner cycle (pentagram)
    \node[draw, circle, fill=black, minimum size=4pt, inner sep=0pt](b1) at (90:0.5) {};
    \node[draw, circle, fill=black, minimum size=4pt, inner sep=0pt](b2) at (162:0.5) {};
    \node[draw, circle, fill=black, minimum size=4pt, inner sep=0pt](b3) at (234:0.5) {};
    \node[draw, circle, fill=black, minimum size=4pt, inner sep=0pt](b4) at (306:0.5) {};
    \node[draw, circle, fill=black, minimum size=4pt, inner sep=0pt](b5) at (18:0.5) {};
    
    % Outer pentagon edges
    \draw (a1) -- (a2);
    \draw (a2) -- (a3);
    \draw (a3) -- (a4);
    \draw (a4) -- (a5);
    \draw (a5) -- (a1);

    % Inner pentagram edges
    \draw (b1) -- (b3);
    \draw (b3) -- (b5);
    \draw (b5) -- (b2);
    \draw (b2) -- (b4);
    \draw (b4) -- (b1);
    
    % Connecting edges between outer and inner nodes
    \draw (a1) -- (b1);
    \draw (a2) -- (b2);
    \draw (a3) -- (b3);
    \draw (a4) -- (b4);
    \draw (a5) -- (b5);
    \end{scope}

\end{tikzpicture}
\caption{Examples of strongly regular graphs, displayed from left to right: the cycle graph $C_5$ ($\srg(5,2,0,1)$), the complete graph 
$K_5$ ($\srg(5,4,3,0)$), the Petersen graph ($\srg(10,3,0,1)$)}
\label{fig:srg_graph_examples}
\end{figure}
Such a graph can also be denoted as $\textsf{srg}(v, k, \lambda,\mu)$ and the four parameters must obey the following relation \citep{biggs1993algebraic}:
\begin{equation}
\label{equality}
{\displaystyle (v-k-1)\mu =k(k-\lambda -1)}.
\end{equation}

\subsection{Regular graphs and WL test}

The WL distinguishability on regular graphs has been studied. This section presents the relationship between the regular/strongly regular graphs and WL tests.

\begin{theorem}
    Regular graphs with the same of number nodes and the same degree are not 1-WL distinguishable.
\end{theorem}
\begin{proof}
    Initially, in the 1-WL algorithm, all nodes are assigned the same color. In each iteration, nodes update their color based on the same multiset of neighboring colors, as all nodes receive identical information. Consequently, the coloring remains uniform across all nodes in every iteration, preventing the algorithm from distinguishing between them.
\end{proof}

\begin{theorem} (\cite{bouritsas2020srg3wl})
    \label{thm:strongly_regular_graphs}
    Strongly regular graphs with the same four parameters are not 3-WL distinguishable.
\end{theorem}

\subsection{Generalization of Regular Graphs}

\citet{CAMERON1980transitive} introduces the extension of regular graphs, namely \(k\)-\emph{transitive} graphs or \(k\)-\emph{isoregular} graphs, as graphs where the number of common neighbors for any \(k\)-tuple of a given isomorphism type remains constant. In this context, we write \emph{simple regular graphs} as \(1\)-isoregular graphs, and \emph{strongly regular graphs} are \(2\)-isoregular. It is known that \(k\)-isoregular graphs are indistinguishable by the \((k+1)\)-dimensional Weisfeiler–Leman test~\citep{douglas2011weisfeilerlehmanmethodgraphisomorphism}.

\section{Proofs}\label{apd:proof}

\begin{lemmacopy}
    {\ref{degree_line_graph}}
    \label{app:degree_line_graph}
    Let $u, v\in V(G)$ such that they are adjacent by an edge $e\in E(G).$ The edge $e$'s corresponding node representation $w_e\in V(\text{L}(G))$ follows \(
    d_{\LL(G)}(w_e) = d_G(u)+d_G(v)-2.\)
\end{lemmacopy}
\begin{proof}
    In the root graph $G$, $u$ is adjacent to $d_G(u)$ edges, and $v$ is adjacent to $d_G(v)$ edges. Among all the edges that $u$, $v$ are adjacent to, $e$ is connected to all other edges by $u$ or $v$ except for $e$ itself. Thus, in the line graph $L(G)$, $w_e$ is adjacent to $d_G(v)-1+d_G(u)-1 = d_G(v)+d_G(u)-2$ other nodes. 
\end{proof}

\begin{corollarycopy}
    {\ref{thm:claw-free}} Let $G$ be a simple and undirected graph, $\LL(G)$ does not contain $K_{1,3}$ as an induced subgraph.
\end{corollarycopy}
\begin{proof}
    This corollary is a direct result of Beineke's forbidden subgraphs. $K_{1,3}$ is one of Beineke's forbidden subgraphs, so line graphs do not contain $K_{1,3}$ as an induced subgraph.
\end{proof}
\begin{theoremcopy}{\ref{thm:equivalence}}
    \label{app:equivalence}
    Let \(\mathcal{G}\) be a set of connected non-claw graphs and \(\mathcal{C}\) be a collection of functions, such that $\forall G_1,G_2\in\mathcal{G}$ such that \(G_1\ncong G_2\), \(\exists h\in \mathcal{C}\), \(h(\LL(G_1))\neq h(\LL(G_2))\). Then, \(\mathcal{C}\) can universally approximate any permutation-invariant function $f:\mathcal{G}\rightarrow \mathbb{R}$.
\end{theoremcopy}

\begin{proof}
    Let $f':\LL(\mathcal{G})\rightarrow \mathbb{R}$ be any permutation-invariant functions on $\LL(\mathcal{G})$. By Theorem \ref{thm:unique_approximation}, since \(\exists{h}\in\mathcal{C}\) can distinguish non-isomorphic graphs in $\LL(\mathcal{G}),$ $\mathcal{C}$ can universally approximate any permutation-invariant functions $f':\LL(\mathcal{G})\rightarrow \mathbb{R}.$ Since $G_1,\,G_2$ are connected and not a claw graph, by Theorem \ref{whitney}, there exists a function $g$ such that $g\circ f' = f$ where $g$ injectively maps $f'$ to $f$.
\end{proof}

\begin{corollarycopy}
    {\ref{cor:n_step_linetransform}}
    Let \(\mathcal{G}\) be a set of connected non-claw graphs and \(\mathcal{C}\) be a collection of functions. If $\forall G_1,G_2\in\mathcal{G}$, \(G_1\ncong G_2\), and \(\forall n \in \mathbb{N}\), \(\exists h\in \mathcal{C}\) such that \(h(\LL^{(n)}(G_1))\neq h(\LL^{(n)}(G_2))\) Then, \(\mathcal{C}\) can universally approximate any permutation-invariant function $f:\mathcal{G}\rightarrow \mathbb{R}$.
\end{corollarycopy}
\begin{proof}
    If $n=1$, the proof is the same as Theorem \ref{thm:equivalence}.

    For $n\geq2$, we can safely assume $\LL^{(n-1)}(G)$ is not isomorphic to $K_{1,3}$ by Corollary \ref{thm:claw-free}. By Theorem \ref{thm:unique_approximation}, since \(\exists{h}\in\mathcal{C}\) can distinguish non-isomorphic graphs in $\LL(\mathcal{G}),$ $\mathcal{C}$ can universally approximate any permutation-invariant functions $f_n:\LL^{(n)}(\mathcal{G})\rightarrow \mathbb{R}.$ Since $\LL^{(n-1)}(G)$ is not a claw graph, by Theorem \ref{whitney}, there exists a function $g_{n}$ such that $g_n\circ f_n = f_{n-1}$ where $g_n$ injectively maps $f_n$ to $f_{n-1}$. By inductively applying Theorem \ref{whitney} we can reach the base case.
\end{proof}

\begin{theoremcopy}
    {\ref{thm:cfi}}
    Line graphs do not include CFI graphs constructed with $X_k$ for $k\geq3$.
\end{theoremcopy}

\begin{proof}
    It can be seen that for $k>2,$ by construction there exists no direct edge connecting $a_i$ and $b_i$, where a single $m_S$ connects to $k$ nodes from $A_k\cup B_k$. Thus, there exists an induced subgraph that is isomorphic to $K_{1,3}$. By Corollary \ref{thm:claw-free}, CFI graphs with $X_k$, $k\geq3$ are excluded by line graph transformation.
\end{proof}

\begin{lemma}
\label{lmm:k_small}
    If a graph $G = \srg(v,k,\lambda,\mu)$ is connected and strongly regular, it holds that $k\leq2$ if and only if $G$ is one of $K_{2}$, $C_{3}$, $C_{4}$, $C_{5}$ and the graph with one node.
\end{lemma}
\begin{proof}
We show the proof by enumerating all the cases when $k\leq 2.$ When $k=0,$ we have only the graph with only one node. When $k=1,$ for the graph to be connected it can only be the case when two nodes are connected by a single edge. When $k=2$, regular graphs are cycles. Cycles with lengths 6 or more are not strongly regular because nodes that are a distance of 2 from each other share exactly one common neighbor, whereas nodes that are 3 or more nodes apart do not share any common neighbors. This inconsistency in the number of common neighbors violates the definition of a strongly regular graph, where the number of common neighbors between nodes must be constant for both adjacent and non-adjacent pairs.
\end{proof}

\begin{theoremcopy}
    {\ref{thm:tri-line-graph-regularity}}Let $G$ be a connected, strongly regular graph which is not the complete graph $K_n$ for $n \geq 2$, and $G$ contains $C_3$ as a subgraph. Then $\LL(G)$ is not strongly regular.
\end{theoremcopy}
\begin{proof} Suppose $G$ is a strongly regular graph denoted as $\text{srg}(v ,k,\lambda,\mu)$ (defined in Appendix \ref{srg_def}). Given that $G$ is not a complete graph, it follows that $\mu \neq 0$, where $\mu$ denotes the number of mutual neighbors between every pair of non-adjacent nodes. 
Also, since $G$ is triangle-containing and not $K_3,$ by Lemma \ref{lmm:k_small}, we have $k\geq3$.
Suppose, for a contradiction, that $\LL(G)$ is strongly regular, which can be denoted as $\text{srg}(v^*,k^*,\lambda^*,\mu^*)$.

By the properties of line graphs, the nodes of $\LL(G)$ correspond to the edges of $G$, hence $v^* = |E(G)|$. Based on the property of regular graphs, we have $ |E(G)|= \frac{vk}{2}=v^* $. In $\LL(G)$, each edge is adjacent to every other edge that shares a node in $G$, thus (also by Equation \ref{degree_line_graph}) we have $k^* = 2(k - 1)$.

Given that $G$ contains triangles, let two edges $e_0$ and $e_1$ in $E(G)$ be part of a triangle shown in Figure \ref{fig:3_mus}. By the definition of line graphs, there exists a node $u \in V(G)$ that is adjacent to $e_0$ and $e_1$. Thus, the third side $e_2$ of the triangle as well as all other $k-2$ edges connected with $u$ are common neighbors in $\LL(G)$. Therefore, for $\LL(G)$ to be strongly regular, $\lambda^*$ must be $k-1$. Given the $k\geq 3$, there exists an edge $e_3$ that does not share a node with $e_2$ on $G$, which corresponds to non-neighboring nodes in $\LL(G)$. This yields $\mu^*$ values ranging from 2 to 4 based on all possible configurations shown in Figure \ref{fig:3_mus}.
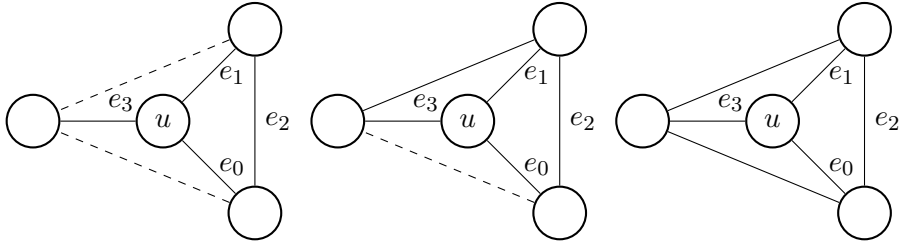
\begin{figure}[ht]
\centering

\begin{tikzpicture}[
                    node distance=1cm,
                    main node/.style={circle, draw=black, thick, minimum size=7mm},
                    indirect edge/.style={dashed}]

  % Nodes
  \node[main node] (0) {$u$};
  \node[main node] (2) [above right=of 0] {};
  \node[main node] (3) [below right=of 0] {};
  \node[main node] (4) [left =of 0] {};
  
  % Edges
  \draw (0) -- (2) node[midway, right] {$e_1$};
  \draw (0) -- (3) node[midway, right] {$e_0$};
  \draw (0) -- (4)node[midway, above right] {$e_3$};
  \draw (2) -- (3)node[midway, right] {$e_2$};
  \draw[indirect edge] (2) -- (4);
  \draw[indirect edge] (3) -- (4);
  
\end{tikzpicture}
\begin{tikzpicture}[>=Stealth, 
                    node distance=1cm,
                    main node/.style={circle, draw=black, thick, minimum size=7mm},
                    indirect edge/.style={dashed}]

  % Nodes
  \node[main node] (0) {$u$};
  \node[main node] (2) [above right=of 0] {};
  \node[main node] (3) [below right=of 0] {};
  \node[main node] (4) [left =of 0] {};
  
  % Edges
  \draw (0) -- (2) node[midway, right] {$e_1$};
  \draw (0) -- (3) node[midway, right] {$e_0$};
  \draw (0) -- (4)node[midway, above right] {$e_3$};
  \draw (2) -- (3)node[midway, right] {$e_2$};
  \draw (2) -- (4);
  \draw[indirect edge] (3) -- (4);
  
\end{tikzpicture}
\begin{tikzpicture}[>=Stealth, 
                    node distance=1cm,
                    main node/.style={circle, draw=black, thick, minimum size=7mm}]

  % Nodes
  \node[main node] (0) {$u$};
  \node[main node] (2) [above right=of 0] {};
  \node[main node] (3) [below right=of 0] {};
  \node[main node] (4) [left =of 0] {};
  
  % Edges
  \draw (0) -- (2) node[midway, right] {$e_1$};
  \draw (0) -- (3) node[midway, right] {$e_0$};
  \draw (0) -- (4)node[midway, above right] {$e_3$};
  \draw (2) -- (3)node[midway, right] {$e_2$};
  \draw (2) -- (4);
  \draw (3) -- (4);
  
\end{tikzpicture}

\caption{Three different cases in the triangle region in $G$ where the dashed line represents no edge exists. In the first case, there is no edge connecting the other end node of $e_3$ to either end node of $e_2$. Only $e_0$ and $e_1$ are the common neighbors in L$(G)$ for the pair $e_2$ and $e_3.$. The second and third cases represent if one or both edges exist, where the edges would also be neighboring nodes in $\LL(G).$}
\label{fig:3_mus}
\end{figure}

We use the following property in Equation \ref{equality} 
where $v^* = \frac{vk}{2}$, $k^* = 2(k - 1)$, $\lambda^* = k - 1$, and $\mu^* = 2, 3, 4$. Evaluating this equation under $\mu^* = 2$ or $\mu^* = 3$ leads to the only positive integer solutions $v = 2, k = 1$ and $v = 3, k = 2$. Both solutions are invalid as we have $k\geq3$. For $\mu^* = 4$, the solutions are $v = k + 1$, which contradicts the assumption that $G$ is not complete. This shows that $\LL(G)$ cannot be strongly regular.

\end{proof}

\begin{theoremcopy}
    {\ref{main_thm}}Let $G$ be a connected graph that is not the cycle graphs $C_3$, $C_4$, or $C_5$. Then, applying at most two line graph transformations to $G$ yields a graph that is not strongly regular.
\end{theoremcopy}
\begin{proof}
The proof is structured based on the characteristics of the graph $G$ and proceeds in five cases:

\textbf{Case 1:} If $G$ is not strongly regular, then by definition, no line graph transformation is required.

\textbf{Case 2:} If $G$ is the graph with one node or $K_2$, we can easily see that $\LL(G)$ and $\LL(\LL(G))$ respectively are not strongly regular.

\textbf{Case 3:} If $G$ is a strongly regular graph that contains triangles and is not complete, we can apply Theorem \ref{thm:tri-line-graph-regularity} to show that $\LL(G)$ is not strongly regular.

\textbf{Case 4:} If a strongly regular graph $G$ is triangle-free with $k\geq3$, each neighbourhood of nodes in $G$ is star-like. Thus, for any two edges $e$ and $e'$ in $G$ that are connected through a node $u$, they share exactly $k-2$ common neighbors (the other edges connected to $u$). Assuming $\LL(G)$ is strongly regular, with parameters $v^*, k^*, \lambda^*, \mu^*$, it follows that $\lambda^* = k-2$, which is nonzero. This implies $\LL(G)$ is triangle-containing. Given that only star graphs have line graphs that are complete, $\LL(G)$ cannot be a complete graph. By Theorem \ref{thm:tri-line-graph-regularity}, we need only one more line graph transformation to have $\LL(\LL(G))$ not strongly regular.

\textbf{Case 5:} By Lemma \ref{lmm:k_small}, we only have complete graphs with $v\geq 4$ left. If $G$ is a complete graph with $v \geq 4$, it follows that the line graph of $G$, $\LL(G)$, can be represented as a strongly regular graph with parameters $\srg\left(\binom {v}{2}, 2(v-2), v-2, 4\right)$ \citep{Harary1969}. In this scenario, $\LL(G)$ contains a triangle implied by the nonzero parameter $v-2$. Applying Theorem \ref{thm:tri-line-graph-regularity}, we can show that the line graph of $\LL(G)$, $\LL(\LL(G))$, is not strongly regular.
\end{proof}

\begin{corollarycopy}
    {\ref{cor:disconnected_equiv}}
Let \(\mathcal{G}\) be a set of graphs such that $\forall G\in\mathcal{G}$, for a fixed $n$, $G$ does not have a component isomorphic to $K_{1,3}$ or $P_k$, $k\leq n$ and Let \(\mathcal{C}\) be a collection of functions. If $\forall G_1,G_2\in\mathcal{G}$ such that \(G_1\ncong G_2\), \(\exists h\in \mathcal{C}\) such that \(h(\LL^{(n)}(G_1))\neq h(\LL^{(n)}(G_2))\). Then, \(\mathcal{C}\) can universally approximate any permutation-invariant function $f:\mathcal{G}\rightarrow \mathbb{R}$.
\end{corollarycopy}
\begin{proof}
   We need to demonstrate that Corollary \ref{cor:n_step_linetransform} applies to disconnected graphs as well. It suffices to show that Theorem \ref{whitney} holds for $G$. For disconnected graphs, the line graph transformation is applied component-wise. Since $G$ contains no component isomorphic to $K_{1,3}$, each component of $G$ is uniquely mapped to a line graph component corresponding to the root component by Theorem \ref{whitney}. Also, under the assumption that no path with a length shorter than $n$ is included in $G$ as a connected component, every component would not vanish after repeated line graph transformation (conversely, we see paths of length $k$ or shorter would be turned to a graph with no nodes after $k$ line graph transformations). Consequently, the line graph of the entire graph $G$ is also unique to its root graph. The rest of the proof proceeds as in Corollary \ref{cor:n_step_linetransform}. 
\end{proof}

\begin{corollarycopy}
    {\ref{cor:disconnected_srg}}
     Let $G$ be a strongly regular graph that is not the cycle graphs $C_3$, $C_4$, $C_5$, or a disjoint group of $C_3$. Then, applying at most two line graph transformations to $G$ yields a graph that is not strongly regular.
\end{corollarycopy}
\begin{proof}
    For a graph $G$ to be disconnected and strongly regular, it could only be the case where the graph is a set of disjoint $K_n$~\citep{biggs1993algebraic}. When $n\leq2$, we can see that at most two line graphs would reduce $G$ to an empty graph. When $n\geq 3$, each component would be reduced to $\srg \left(\binom {v}{2}, 2(v-2), v-2, 4\right)$. The overall graphs would not be strongly regular because one node in a component does not share any common neighbors with nodes in other components. When $G$ is a set of $C_3$ (i.e. $K_3$), we have $\LL^{(n)}(G)\cong G.$
\end{proof}

\section{Experimental details}\label{apd:second}
\subsection{Code availability}
We open-source our code to replicate the experiments on \href{https://github.com/lukeyf/graphs-and-lines}{GitHub}. The details about the experimental setup and training parameters can be found in the GitHub repository.
\subsection{Dataset}
Our experiments were conducted on the BREC dataset \citep{wang2024brec} with the sections \textbf{Regular Graphs}, and \textbf{CFI Graphs}. A summary of each category is provided below.

We selected 120 pairs of regular graphs, which can be further divided into simple regular graphs, strongly regular graphs, and 4-vertex condition graphs:
\begin{itemize}
    \item \textbf{Simple regular graphs}: We selected 50 pairs of simple regular graphs, each with 6 to 10 nodes, by randomly choosing pairs with identical parameters.
    \item \textbf{Strongly regular graphs}: This subset includes 50 pairs of strongly regular graphs with node counts ranging from 16 to 35. The graphs were sourced from databases such as \href{http://www.maths.gla.ac.uk/}{SR Graphs} and \href{http://users.cecs.anu.edu.au/~bdm/data/graphs.html}{BDM Graphs}.
    \item \textbf{4-vertex condition graphs}: A set of 20 pairs of 4-vertex condition graphs was selected from the \href{http://math.ihringer.org/srgs.php}{4-vertex Condition Graph Database} with parameters $\srg(63, 30, 13, 15).$
\end{itemize}
Note that 4-vertex condition graphs are a specific subtype of strongly regular graphs, and as such, we classify them within the same category as other strongly regular graphs in our results.

We selected the 100 pairs of graphs in the BREC dataset. The backbone graphs ranged from 3 to 7 nodes. The dataset contains:
\begin{itemize}
    \item 60 pairs of 1-WL-indistinguishable CFI graphs,
    \item 20 pairs of 3-WL-indistinguishable CFI graphs, and
    \item 20 pairs of 4-WL-indistinguishable CFI graphs.
\end{itemize}

\subsection{Provably Powerful Neural Networks}\label{app:ppgn}
To address the limitations of 1-WL expressiveness, recent research has focused on constructing GNNs that match or surpass the expressive power of higher-order WL tests. In this section, we discuss the provably powerful neural network (PPGN) proposed by \citet{maron2020provablypowerfulgraphnetworks} with 3-WL expressiveness that can distinguish between non-isomorphic graphs that are indistinguishable by 1-WL.

The construction of a 3-WL expressive GNN involves three main components:
\begin{enumerate}
    \item \textbf{Input representation}: The graph $G = (V, E, d)$, where $V$ is the set of nodes, $E$ is the set of edges, and $d$ represents node features (or colors), is represented as a tensor $B \in \mathbb{R}^{n^2 \times (e+1)}$, where $n$ is the number of nodes, $e$ is the number of features, and the last channel of $B$ encodes the graph adjacency matrix.
    
    \item \textbf{Network layers}: The key operations of the GNN are organized into \emph{blocks}, each consisting of:
    \begin{itemize}
        \item A \emph{Multi-layer Perceptron (MLP)} applied independently to each feature of the input tensor. This is denoted as $m_i$ where $i \in \{1, 2, 3\}$. Each MLP transforms the input features.
        \item A \emph{matrix multiplication} between the transformed feature tensors. Let $X \in \mathbb{R}^{n \times n \times a}$ be the input tensor to the block, where $a$ is the feature dimension. The matrix multiplication is performed between the transformed tensors: 
        \[
        W_{:, :, j} = m_1(X)_{:, :, j} \cdot m_2(X)_{:, :, j} \quad \forall j \in \{1, \dots, b\}.
        \]
        The output tensor of the block consists of the concatenated MLP transformation $m_3(X)$ and the result of matrix multiplication $W$.
    \end{itemize}

    \item \textbf{Invariant and equivariant layers}: The model employs invariant and equivariant layers to ensure that its operations respect the permutation symmetry of the graph. The matrix multiplication operation is equivariant to permutations, ensuring that the model’s output is invariant to node reordering. This structure ensures that the GNN can distinguish between non-isomorphic graphs that 1-WL cannot.
\end{enumerate}

The GNN described above is provably as expressive as the 3-WL test. Formally, the following result holds:
\begin{itemize}
    \item For any two graphs $G$ and $G'$ that can be distinguished by the 3-WL graph isomorphism test, there exists a 3-WL expressive PPGN $F$ such that $F(G) \neq F(G')$.
    \item Conversely, if $G$ and $G'$ are isomorphic, then $F(G) = F(G')$ for any 3-WL expressive PPGN $F$.
\end{itemize}

\subsection{Complexity analysis}

In this section, we present the space and time complexity analysis associated with applying the line graph transformation in the context of \( k \)-WL tests. The \( k \)-WL test has a space complexity of \( O(n^k) \) and a time complexity of \( O(n^{k+1}) \) \citep{feng2023wlcomplexity}, where $n$ is the number of nodes in a graph $G$, i.e., $n = |V(G)|$. After performing the line graph transformation, the number of nodes in the transformed graph corresponds to the number of edges in the original graph. Consequently, for structures like paths and cycles, the space complexity remains \( O(n^k) \) and the time complexity \( O(n^{k+1}) \).

For a \( d \)-regular graph, the number of nodes in the line graph is \( \frac{dn}{2} \). Therefore, the space complexity becomes \( O(d^k n^k) \), and the time complexity is \( O(d^{k+1} n^{k+1}) \). In the worst-case scenario, when the graph is dense—such as in a complete graph with \( \frac{n(n-1)}{2} \) edges—the space complexity increases to \( O(n^{2k}) \) and the time complexity to \( O(n^{2k+2}) \). 

The empirical comparison of time consumption on regular graphs is presented in Figure \ref{fig:speed}. A consistent increase in time consumption by orders of magnitude is observed. We present the subset of 4-vertex condition strongly regular graphs separately, as their significantly larger size distinguishes them from other strongly regular graphs in the BREC dataset. In the PPGN comparison, the line graph generally requires more time compared to the root graph. However, the results vary depending on the specific hyperparameter and epoch settings of the PPGN.

\begin{figure}
    \centering
    \includegraphics[width=1\linewidth]{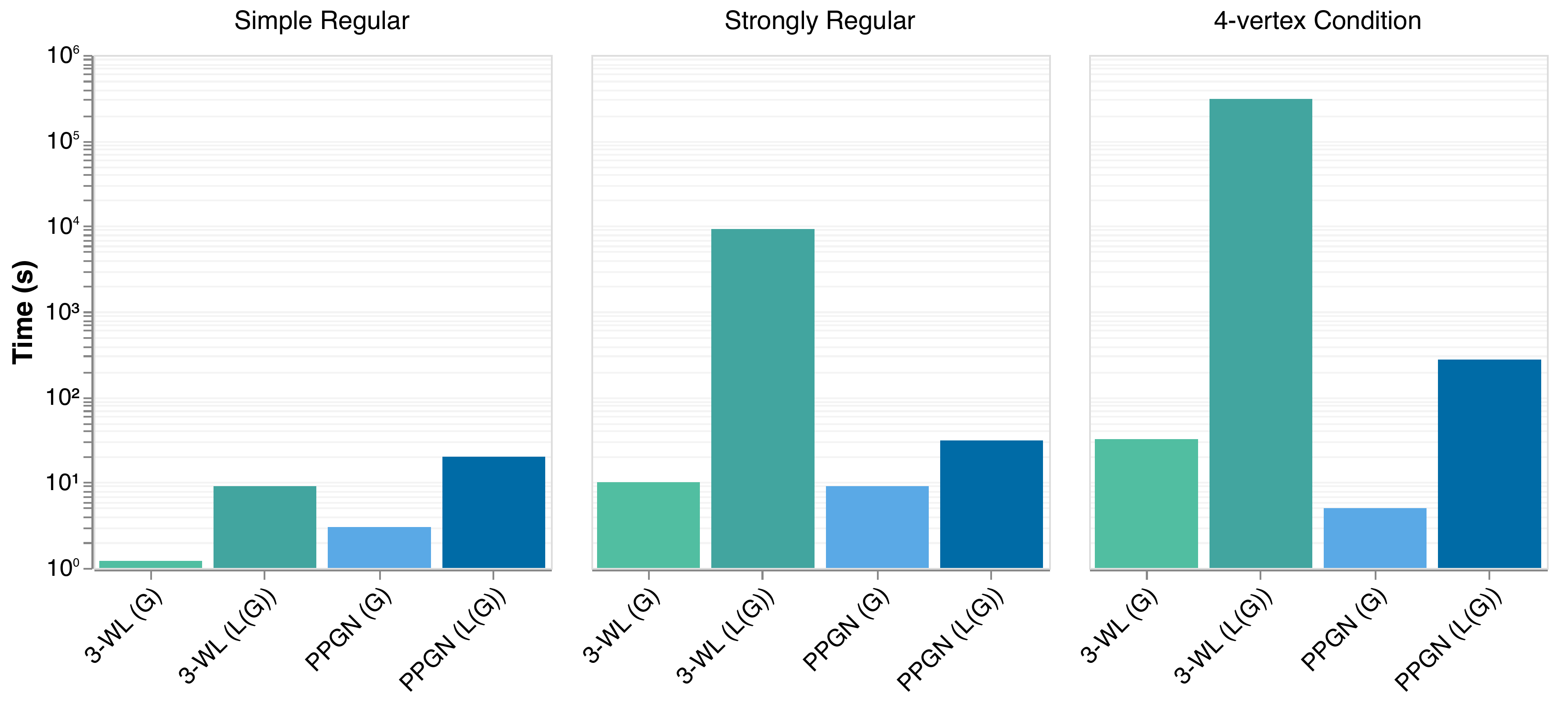}
    \caption{Time consumption on simple regular, strongly regular, and 4-vertex condition strongly regular subsets.}
    \label{fig:speed}
\end{figure}
 
\subsection{Experimental setup and parameters}

For training the PPGN, we employed the Adam optimizer with a learning rate of 0.0001 and a weight decay of 0.0001. The loss function used was \texttt{CosineEmbeddingLoss}. The model architecture consisted of 5 layers with an inner embedding dimension of 32. For the root graph, a batch size of 32 was used, whereas for the line graph, a smaller batch size of 4 was applied due to memory limitations.
\subsection{Compute resources}

For 3-WL and 4-WL tests, we used 4-core CPU clusters to run the analysis. For the PPGN experiment, we used an additional 40 GB A100 GPU to accelerate the training.

\end{document}